\documentclass{article}

\usepackage[final]{neurips_2023}

\usepackage[utf8]{inputenc} 
\usepackage[T1]{fontenc}    
\usepackage{hyperref}       
\usepackage{url}            
\usepackage{booktabs}       
\usepackage{amsfonts}       
\usepackage{nicefrac}       
\usepackage{microtype}      
\usepackage{xcolor}         

\usepackage{enumitem}
\usepackage{caption}

\usepackage{amsmath}
\usepackage{amssymb}
\usepackage{mathtools}
\usepackage{amsthm}

\usepackage{graphicx}
\usepackage{subcaption}
\usepackage{asymptote}
\usepackage{tikz}
\usepackage{makecell}
\usepackage{pgfplotstable}
\usepackage{pgfplots}
\usepackage{wrapfig}

\newcommand{\PreserveBackslash}[1]{\let\temp=\\#1\let\\=\temp}
\newcolumntype{C}[1]{>{\PreserveBackslash\centering}p{#1}}
\newcolumntype{R}[1]{>{\PreserveBackslash\raggedleft}p{#1}}
\newcolumntype{L}[1]{>{\PreserveBackslash\raggedright}p{#1}}

\newcommand{\Xcal}{\mathcal{X}}
\newcommand{\br}[1]{\left({#1}\right)}
\newcommand{\bs}[1]{\boldsymbol{#1}}
\newcommand{\bsx}{\boldsymbol{x}}

\newcommand{\bsc}{\boldsymbol{c}}
\newcommand{\ch}{\mathrm{CONV}}
\newcommand{\preimage}{f^{-1}(\mathcal{S}_{\text{out}})}
\newcommand{\Sout}{\mathcal{S}_{\text{out}}}
\newcommand{\Sin}{\mathcal{S}_{\text{in}}}
\newcommand{\Sover}{\mathcal{S}_{\text{over}}}
\newcommand{\SLP}{\mathcal{S}_{\text{LP}}}
\newcommand{\SMILP}{\mathcal{S}_{\text{MILP}}}
\newcommand{\ReLU}{\mathsf{ReLU}}

\newcommand{\abcrown}{\alpha,\!\beta\text{-CROWN}} 

\theoremstyle{thm}
\newtheorem{theorem}{Theorem}

\theoremstyle{definition}

\theoremstyle{remark}

\theoremstyle{definition}
\newtheorem{exmp}{Example}

\ifdef{\hideNotes}{%
\newcommand{\huan}[1]{}
\newcommand{\suhas}[1]{}
\newcommand{\chris}[1]{}
\newcommand{\dvij}[1]{}
\newcommand{\zico}[1]{}}
{
\newcommand{\huan}[1]{\textcolor{blue}{Huan: #1}}
\newcommand{\suhas}[1]{\textcolor{purple}{Suhas #1}}
\newcommand{\chris}[1]{\textcolor{orange}{Christopher #1}}
\newcommand{\dvij}[1]{\textcolor{green}{Dj #1}}
\newcommand{\zico}[1]{\textcolor{red}{[Zico: #1]}}
}

\title{Provably Bounding Neural Network Preimages}

\author{%
  Suhas Kotha$^*$ \\
  Carnegie Mellon\\
  \texttt{suhask@andrew.cmu.edu} \\
  \And
  Christopher Brix$^*$ \\
  RWTH Aachen\\
  \texttt{brix@cs.rwth-aachen.de} \\
  \AND
  Zico Kolter \\
  Carnegie Mellon \\
  Bosch Center for AI \\
  \texttt{zkolter@cs.cmu.edu}
  \And
  Krishnamurthy (Dj) Dvijotham$^\dagger$ \\
  Google DeepMind\\
  \texttt{dvijothamcs@gmail.com}
  \And
  Huan Zhang$^\dagger$ \\
  UIUC\\
  \texttt{huan@huan-zhang.com}
}

\begin{document}

\maketitle

\def\thefootnote{\phantom{}}\footnotetext{* Equal contribution, ${}^\dagger$ Equal advising}\def\thefootnote{\arabic{footnote}}
\def\thefootnote{\phantom{}}\footnotetext{Instructions for reproducing our results are available at \href{https://github.com/kothasuhas/verify-input}{\texttt{https://github.com/kothasuhas/verify-input}}}\def\thefootnote{\arabic{footnote}}

\begin{abstract}
Most work on the formal verification of neural networks has focused on bounding the set of outputs that correspond to a given set of inputs (for example, bounded perturbations of a nominal input).
However, many use cases of neural network verification require solving the inverse problem, or over-approximating the set of inputs that lead to certain outputs.
We present the INVPROP algorithm for verifying properties over the preimage of a linearly constrained output set, which can be combined with branch-and-bound to increase precision.
Contrary to other approaches, our efficient algorithm is GPU-accelerated and does not require a linear programming solver.
We demonstrate our algorithm for identifying safe control regions for a dynamical system via backward reachability analysis, verifying adversarial robustness, and detecting out-of-distribution inputs to a neural network.
Our results show that in certain settings, we find over-approximations over $2500\times$ tighter than prior work while being $2.5\times$ faster. By strengthening robustness verification with output constraints, we consistently verify more properties than the previous state-of-the-art on multiple benchmarks, including a large model with 167k neurons in VNN-COMP 2023. Our algorithm has been incorporated into the $\alpha,\!\beta$-CROWN verifier, available at~\texttt{\url{https://abcrown.org}}.
\end{abstract}



\section{Introduction}\label{sec:intro}
Applying neural networks to safety-critical settings often requires reasoning about constraints on the inputs and outputs of a neural network. For example, for a physical system controlled by a neural network policy, it is of interest to understand which initial states will lead to an unsafe state such as colliding with an obstacle. Formal verification of neural networks seeks to provide provable guarantees demonstrating that networks satisfy formal specifications on their inputs and outputs.
Most work to date has focused on developing algorithms that can bound the outputs of a neural network given constraints on the inputs, which can be used for applications such as analyzing the robustness of a neural network to perturbations of a given input \citep{wong2018provable, dvijotham2018dual, zhang2018efficient,raghunathan2018certified, 8418593}.

In this work, we address the inverse problem of over-approximating a neural network's preimage: given a set of outputs $\Sout$ (described by linear constraints on the network output), we seek to find a set that provably contains all inputs that lead to such outputs. For example, for our neural network policy, this would correspond to the states that collide with an obstacle one step in the future. Though the verification problem is already challenging due to non-convexity and high dimensionality, this new problem is even more difficult since neural networks are generally not invertible. 

Specifically, representative efficient verifiers (such as state-of-the-art bound-propagation-based methods \citep{zhang2022general}) can only compute bounds utilizing constraints on the input and critically depend on having tight bounds on intermediate activations. In the setting of this paper, however, the bounds derived from the input constraints are almost vacuous, since the only constraints on the input are that it should be from the valid input domain. We efficiently solve this problem by significantly generalizing the existing bound propagation-based verification framework, allowing one to leverage output constraints when tightening the intermediate activations. Our contributions are as follows:

\begin{itemize}[wide]
    \item We develop an effective bound propagation framework,
    Inverse Propagation for Neural Network Verification (INVPROP), for the \emph{inverse verification problem} for neural networks, i.e., the problem of over-approximating the set of inputs that leads to a given set of outputs. Importantly, INVPROP requires no linear programming solver and can compute bounds on any intermediate layer.
    \item We unify INVPROP and traditional bound propagation into a more general verification framework, allowing us to connect our method to standard tools, such as the state-of-the-art bound propagation tool $\alpha,\!\beta$-CROWN \citep{zhang2018efficient, xu2021fast, wang2021beta, zhang2022general, brix2023vnncomp, mueller2023international}. Our contribution allows $\alpha,\!\beta$-CROWN to tighten intermediate bounds with respect to output constraints, which could not be done by the original tool.
    
    \item We demonstrate that tight inverse verification requires multiple iterative refinements of intermediate bounds. While layer bounds in standard bound propagation only depend on the bounds of their predecessors, INVPROP incorporates bounds of \emph{all} layers in the network.
    \item We improve the state of the art on a control benchmark \citep{rober2022hybrid, rober2022backward} by providing $2500\times$ tighter bounds, $2.5\times$ faster, for a Double Integrator and $257\times$ tighter bounds, $3.29\times$ faster, for a 6D Quadrotor. Furthermore, we demonstrate that INVPROP can strengthen robustness verification with output constraints and verify more robustness properties in less time compared to existing tools. Finally, we demonstrate its applicability in OOD detection. 
\end{itemize}

\section{Background and Problem Setup}\label{sec:setup}

\subsection{Notation}\label{sec:notation}

We use $[L]$ for $L \in \mathbb{N}$ to refer to the set $\{1, 2, \ldots, L\}$, $\mathbf{W}_{:,j}^{(i)}$ to refer to column $j$ of the matrix $\mathbf{W}^{(i)}$, $[\cdot]_+$ to refer to $\max(0, \cdot)$, and $[\cdot]_-$ to refer to $-\min(0, \cdot)$.
We use boldface symbols for vectors and matrices (such as $\bsx^{(i)}$ and $\mathbf{W}^{(i)}$) and regular symbols for scalars (such as $x^{(i)}_j$). We use $\bs{x} \odot \bs{y}$ to denote element-wise multiplication of vectors $\bs{x}, \bs{y}$.

We define an $L$ layer ReLU neural network by its weight matrices $\mathbf{W}^{(i)}$ and bias vectors $\mathbf{b}^{(i)}$ for $i \in [L]$. The output of the neural network for the input $\hat{\bsx}^{(0)}$ from a bounded input domain $\mathcal{X}$ is computed by alternately applying linear layers $\bsx^{(i)} = \mathbf{W}^{(i)}\hat{\bsx}^{(i-1)} + \mathbf{b}^{(i)}$ and ReLU layers $\hat{\bsx}^{(i)} = \max(\mathbf{0}, \bsx^{(i)})$ until we receive the output $\bsx^{(L)}$ (which we refer to as the logits). Note that we treat softmax as a component of the loss function, not the neural network.

\subsection{Problem Statement}\label{sec:statement}

Given a neural network $f: \Xcal \subseteq \mathbb{R}^{\text{in}} \rightarrow \mathbb{R}^{\text{out}}$ and an output constraint $\Sout \subseteq \mathbb{R}^{\text{out}}$, we want to compute $\preimage \subseteq \mathcal{X}$.
Since precisely computing or expressing $\preimage$ is an intractable problem in general, we strive to compute a tight over-approximation $\Sover$ such that $\preimage \subseteq \Sover$. In particular, we target the convex hull of the preimage via a cutting-plane representation. 
$\Sout$ will be defined by a set of linear constraints parameterized by $\mathbf{H}f\left(\bs{x}\right) + \mathbf{d} \leq \mathbf{0}$ in this work.

\subsection{Applications}
\label{sec:application}

\paragraph{Backward Reachability Analysis for Neural Feedback Loops.}

Establishing safety guarantees for neural network policies is a challenging task. 
One problem of interest is to find a set of initial states that does not reach a particular set of future states under the neural network policy. This can be helpful in collision avoidance or control with safety constraints. 
For example, consider a discrete-time double integrator controller \citep{hu2020doubleintegrator} where the state at time $t+1$ can be directly computed based on state at time $t$ following the equation
\[\bs{x}_{t+1} = f(\bs{x}_t) = \begin{bmatrix} 1 & 1 \\ 0 & 1\end{bmatrix}\bs{x}_t + \begin{bmatrix}0.5 \\ 1\end{bmatrix}\pi(\bs{x}_t)\]
with policy $\pi : \mathbb{R}^2 \to \mathbb{R}$.
If there is an obstacle in the room covering the region $[4.5, 5.0] \times [-0.25, 0.25]$, it is of interest to understand which states will enter the unsafe region in the next time-step.
We can represent this obstacle set with linear constraints that define $\Sout$.

$\preimage$ denotes the set of states $\bs{x_t}$ such that $\bs{x_{t+1}}$ lies in the unsafe region given the control policy $\pi$.
Overapproximating this set allows us to define the set of states to avoid one timestep in advance. We can compose $f^{-1}$ with itself $t$ times to obtain the set of initial states where $\pi$ would drive the system to the unsafe set after $t$ steps.

\paragraph{Robustness Verification}
Adversarial robustness is one classic problem for neural network verification where verifiers assess whether all perturbations of an input yield the same classification.
INVPROP can be used to speed up this verification by exploiting an implicit output constraint. We follow \cite{wang2021beta} and use the canonical form of verifying a property by proving that $\min_{x \in \Xcal} f(x) > 0$. $f(x)$ will be negative if and only if $x$ is an adversarial example.
This minimization problem can be rewritten as $\inf_{x \in X \land f(x) \leq 0} f(x)$ \citep{pengfei2021improving}.
If no adversarial example exists, the infimum is computed over an empty set and returns positive infinity, indicating local robustness.
If at least one adversarial example exists, the result is guaranteed to be negative.
With the output constraint $f(x) \geq 0$, we only analyze inputs that yield an incorrect classification, reducing the search space and tightening verification bounds.

\paragraph{Determining what inputs lead to confident predictions.}\label{sec:ood} 

It is challenging to train a classifier that knows when it doesn't know. A simple and effective approach is to train with an additional logit representing OOD abstention. Labeled data for this additional class is generated via outlier exposure, or adding synthetic training data known to be OOD \citep{hendrycks2018deep, chen2021atom}.

Consider the logits $y$ produced by a binary classifier trained in this manner. The classifier can classify in-distribution data by comparing $y_0$ and $y_1$, the logits corresponding to the two in-distribution labels. The quantity $\max(y_0, y_1) - y_2$, known as the logit gap, can be used to identify out-of-distribution data (Fig.~\ref{calibrated-toy-ood} in Appendix~\ref{sec:ooddetectionfigure}). Can the logit gap correctly identify data far from the training data as being OOD? To answer this quesiton, we can compute the preimage of the set $$\Sout = \{\bs{y} : \max(y_0, y_1) \geq y_2\}$$ We detail how to model this set with linear constraints in Appendix~\ref{sec:encodemax}.
INVPROP enables us to over-approximate $\preimage$, answering whether the classifier learned to correctly identify OOD data.

\section{Approach}\label{sec:approach}


\subsection{Convex over-approximation}

Suppose we want to find a convex over-approximation of $\preimage$. The tightest such set is its convex hull, which is the intersection of all half-spaces that contain $\preimage$ \citep{boyd2004convex}. This intersection is equivalent to
$\bigcap_{\bsc\in \mathbb{R}^{\text{in}}} \{\bsx : \bsc^{\top}\bsx \geq \min_{f(\bs{x'})\in{\Sout}}\bsc^{\top}\bs{x'}\},$ which means that we can build an over-approximation by taking this intersection for finitely many $\bsc$. Furthermore, replacing the minimization problem with a lower bound to its value still yields a valid half-space, where a tighter bound yields a tighter approximation.

We focus on convex over-approximations for two reasons: 
First, checking that the preimage satisfies a linear constraint $\boldsymbol{c}^\top \bsx + d \geq 0$ is equivalent to minimizing the linear function $\bsx \mapsto \boldsymbol{c}^\top \bsx$ over the preimage, which is in turn equivalent to minimizing the linear function over $\ch\br{\preimage}$ \citep{boyd2004convex}.
Second, the convex hull and its tractable over-approximations are conveniently represented as intersections of linear constraints on the preimage, each of which can be quickly computed after a single precomputation phase to tighten bounds on intermediate neurons using the INVPROP algorithm outlined in the next section.

For the above reasons, we will focus on solving the following constrained optimization problem. 
\begin{align}
\min _{\bs{x}} \quad & \bs{c}^{\top}\bs{x} ;
\quad \quad \text { s.t. } \bs{x} \in \Xcal; \quad f\br{\bs{x}} \in \Sout
\label{eq:advopt}
\end{align}
Note that this differs from the widely studied forward verification problem, which can be phrased as 
\begin{align*}
\min _{\bs{x} \in \Sin} \quad & \bs{c}^{\top}f\br{\bs{x}}
\end{align*}
where $\Sin$ is a set representing constraints on the input, and the goal of the verification is to establish that the output $f\br{\bs{x}}$ satisfies a linear constraint of the form $\bs{c}^\top f\br{\bs{x}} \geq d$ for all $x \in \Sin$.

\subsection{The INVPROP Algorithm}

As a brief overview of our method, we first construct the Mixed Integer Linear Program (MILP) for solving this optimization problem, which generates the over-approximation $\mathcal{S}_{\text{MILP}}$ when solved for finitely many $\bs{c}$. Then, we relax the MILP to a Linear Program (LP), which can construct the over-approximation $\mathcal{S}_{\text{LP}}$. Finally, we relax the LP via its Lagrangian dual, which will be used to construct our over-approximation $\Sover$. This chain of relaxations is visualized in Figure \ref{overapprox-viz}.

\begin{figure}[t]
\begin{center}
\begin{minipage}[c]{0.47\textwidth}
    \includegraphics[width=\textwidth] {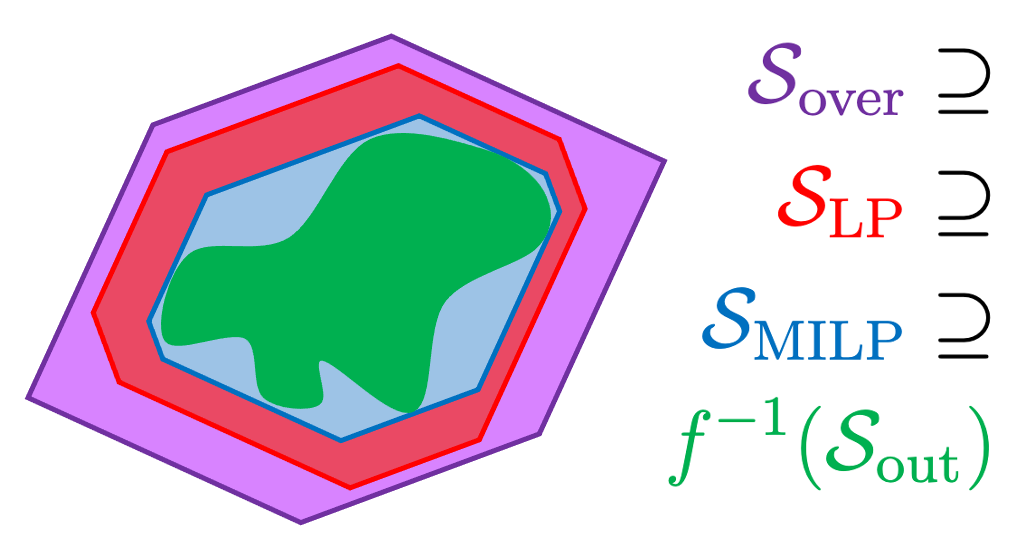}
  \end{minipage}\hfill
  \begin{minipage}[c]{0.5\textwidth}
    \caption{
       \textbf{Visualization of relaxations.} The inner green region depicts the true $\preimage$, the blue relaxation depicts the intersection of finite half-spaces solved via MILP, the red relaxation displays the same via LP, and the purple relaxation displays the same via INVPROP.
       Though this diagram displays looseness, we provide a comprehensive methodology to reduce the error in all three relaxations up to arbitrary precision (Section \ref{sec:branching}). 
    } \label{overapprox-viz}
  \end{minipage}
\end{center}
\end{figure}

\paragraph{The Mixed Integer Linear Programming (MILP) Formulation}

For feed-forward ReLU neural networks, the non-linearities from the max operator can be encoded via integer variables, and this problem admits a MILP encoding similar to prior work in adversarial robustness \citep{tjeng2017evaluating}.
Problem \eqref{eq:advopt} is equivalent to:
\begin{subequations}
\begin{align}
\min _{\bs{x}, \hat{\bs{x}}} \quad & \bs{c}^{\top}\bs{x} \\
\text { s.t. } \quad &
\bs{x} \in \mathcal{X} ; 
\quad  \hat{\bs{x}}^{(0)} = \bs{x} ; 
\quad \mathbf{H}\bs{x}^{(L)} + \mathbf{d} \leq \mathbf{0} ; \\
\quad & \bs{x}^{(i)} = \mathbf{W}^{(i)} \hat{\bs{x}}^{(i-1)}+\mathbf{b}^{(i)}
\quad i \in[L] ; \\
\quad & \hat{\bs{x}}^{(i)} = \max(0, \bs{x}^{(i)}) ;
\quad i\in[L-1]
\end{align}\label{eq:MILP}    
\end{subequations}
\paragraph{The Linear Programming (LP) Formulation}\label{sec:lp}

Unfortunately, finding an exact solution to this MILP is NP-complete \citep{katz2017reluplex}. To sidestep the intractability of exact verification, we can compute lower bounds for this program via its convex relaxation. We consider bounds on the outputs of intermediate layers: 
\[l^{(i)}_j \leq x^{(i)}_j \leq u^{(i)}_j\]
Based on these bounds, we can take the ReLU triangle relaxation \citep{ehlers2017formal} to get the LP
\begin{subequations}
\begin{align}
\min_{\bs{x}, \hat{\bs{x}}} \quad & \bs{c}^{\top}\bs{x} \label{eq:lp_obj} \\
\text { s.t. } \quad &
\bs{x} \in \mathcal{X} ;
\quad \bs{l}^{(0)} \leq \bs{x} \leq \bs{u}^{(0)} ;
\quad \hat{\bs{x}}^{(0)} = \bs{x} ;
\quad \mathbf{H}\bs{x}^{(L)} + \mathbf{d} \leq \mathbf{0} \label{eq:lp_output_constraint} \\
& \bs{x}^{(i)} =\mathbf{W}^{(i)} \hat{\bs{x}}^{(i-1)}+\mathbf{b}^{(i)} \\
& \bs{0} \leq \hat{\bs{x}}^{(i)} ;
\quad \bs{x}^{(i)} \leq \hat{\bs{x}}^{(i)} ;
\quad \hat{\bs{x}}^{(i)} \leq \frac{\bs{u}^{(i)}}{\bs{u}^{(i)} - \bs{l}^{(i)}} \odot \left(\bs{x}^{(i)} - \bs{l}^{(i)}\right) ;
\quad i \in [L]
\end{align}\label{eq:LP}    
\end{subequations}

where the bounds $\bs{l}^{(0)} \leq \bs{x} \leq \bs{u}^{(0)}$ are either the bounds implicit in $\Xcal$, or a refinement of these obtained via previous rounds of bound propagation or input branching (see Algorithm \ref{alg:invprop} for details).  

Most efficient neural network verifiers do not solve an LP formulation of verification directly because LP solvers are often slow for problem instances involving neural networks. Instead, the bound propagation framework~\citep{zhang2018efficient,wang2021beta,zhang2022general} is a practical and efficient way to lower bound the LP relaxation of the forward verification problem.
However, there are \emph{two major roadblocks} to applying existing methods here: Typical bound-propagation cannot directly handle the output constraints (Eq. \ref{eq:lp_output_constraint}) and the objective involving input layer variables (Eq. \ref{eq:lp_obj}).
This is true for optimizing bounds on the input, intermediate layers, and the output layer, all of which need to be iteratively tightened as demonstrated later.



\paragraph{The Inverse Propagation (INVPROP) Formulation}

By changing the LP above to optimize the quantity $\hat{x}^{(0)}_j$ (input layer) or $x^{(i)}_j$ (intermediate layers) for $i\in[L-1]$, the bounds for the $j$-th neuron of layer $i$ can be tightened in separate LP calls.
However, this program is too expensive to be run multiple times for each neuron. 

Inspired by the success of CROWN-family neural network verifiers \citep{zhang2018efficient, xu2021fast, wang2021beta, zhang2022general}, we efficiently lower bound the solution of the LP by optimizing its Lagrangian dual. This dual is highly structured~\citep{wong2018provable}, allowing us to bound input half-spaces $\bs{c}^{\top}\bsx$ and intermediate bounds $l^{(i)}_j, u^{(i)}_j$ by closed-form expressions of the dual variables.
Our main generalization is the ability to optimize input or intermediate layer bounds with the output constraints \textbf{after} them in the neural network, as shown in the following theorem. 
\begin{theorem}[Bounding input half-spaces]\label{thm:gamma-crown}

Given an output set $\Sout = \{\bs{y} : \mathbf{H}\bs{y} + \mathbf{d} \leq \mathbf{0}\}$ and vector $\bs{c}$, $g_{\bs{c}}(\bs{\alpha}, \bs{\gamma})$ is a lower bound to the linear program in \eqref{eq:LP} for $\bs{0} \leq \bs{\alpha} \leq \bs{1}$, $\bs{\gamma} \geq \bs{0}$, and $g_{\bs{c}}$ defined via 
$$
\begin{aligned}
g_{\bs{c}}(\bs{\alpha}, \bs{\gamma}) = &\left[\bs{c}^{\top} - \boldsymbol{\nu}^{(1) \top} \mathbf{W}^{(1)}\right]_+\bs{l}^{(0)}-\left[\bs{c}^{\top} -\boldsymbol{\nu}^{(1) \top} \mathbf{W}^{(1)}\right]_-\bs{u}^{(0)} \\ &-\sum_{i=1}^L \boldsymbol{\nu}^{(i) \top} \mathbf{b}^{(i)} + \sum_{i=1}^{L-1} \sum_{j \in \mathcal{I}^{\pm(i)}}\left[\frac{u^{(i)}_{j}l^{(i)}_j [\hat{\nu}^{(i)}_j]_{+}}{u^{(j)}_{i} - l^{(j)}_i} \right]
\end{aligned}
$$
where every term can be directly recursively computed via
$$
\begin{aligned}
\mathcal{I}^{-(i)} &= \{j : u^{(i)}_j \leq 0\} ; 
\quad \mathcal{I}^{+(i)} = \{j : l^{(i)}_j \geq 0\} ; 
\quad \mathcal{I}^{\pm(i)} = \{j : l^{(i)}_j < 0 < u^{(i)}_j\} \\ 
\boldsymbol{\nu}^{(L)} &= -\boldsymbol{\gamma} ;
\quad \hat{\nu}_j^{(i)} = \boldsymbol{\nu}^{(i+1) \top} \mathbf{W}_{:, j}^{(i+1)} \quad \forall i \in [L-1] \\
\nu_j^{(i)} &= \left\{\begin{array}{lr}
\hat{\nu}_j^{(i)}, &j \in \mathcal{I}^{+(i)} \\
0, & j \in \mathcal{I}^{-(i)} \\
\frac{u^{(i)}_{j}}{u^{(j)}_{i} - l^{(j)}_i}[\hat{\nu}^{(i)}_j]_{+} - \alpha^{(i)}_j[\hat{\nu}^{(i)}_{j}]_{-}, & j \in \mathcal{I}^{\pm(i)}
\end{array}\right. \quad \forall i \in [L-1]
\end{aligned}
$$
\end{theorem}
\begin{proof}
Full proof is presented in Appendix \ref{sec:theorem-proof}.
\end{proof}


In Appendix \ref{sec:intermediate-bounds-theorem}, we show how to bound intermediate neurons in the network using a similar approach. Since the intermediate bounds might depend on bounds on neurons in subsequent layers (due to the output constraint), we cannot simply optimize bounds in a single forward pass layer by layer, unlike prior work such as $\alpha$-CROWN~\citep{xu2021fast}. Instead, we must iteratively tighten intermediate layer bounds with respect to the tightest bounds on all neurons computed thus far. This iterative approach can tighten the initially loose bounds by several orders of magnitude, as shown in Figure \ref{fig:iterating}.
After performing this procedure once, the intermediate bounds can be used to tightly lower bound $\bs{c}^\top\bs{x}$ for any $\bs{c}$ via Theorem~\ref{thm:gamma-crown}. Therefore, this computation can be shared across all the constraints $\bs{c}$ we use to describe $\Sover$. Our algorithm can be expressed in terms of forward/backward passes through layers of the neural network and implemented via standard deep learning modules in libraries like PyTorch \citep{NEURIPS2019_9015}. Since all of the operations are auto-differentiable, we can tighten our lower bound using standard gradient ascent (projected by the dual variable constraints).


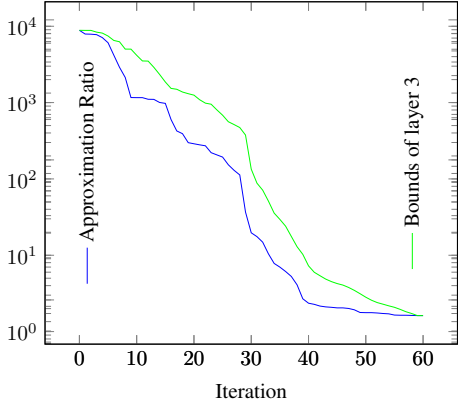
\begin{figure}[b!]
\begin{center}
\begin{minipage}[c]{0.45\textwidth}
\scalebox{0.8}{
    \begin{tikzpicture}[scale=1.0]
     \begin{axis}[
       xlabel={Iteration},
       ylabel={\ref*{line:ar} Approximation Ratio},  
       ymode=log,
       ylabel near ticks, yticklabel pos=left,   
ylabel shift = -50 pt
]
     \addplot[color=blue] table [y=AR, x=iteration]{icml2023/assets/bp_10.dat}; \label{line:ar}
     \end{axis}
     \begin{axis}[
       ylabel={\ref*{line:interm} Bounds of layer 3},
       ymode=log,
       ylabel near ticks, yticklabel pos=right,        
ylabel shift = -50 pt,
]
     \addplot[color=green] table [y=interm, x=iteration]{icml2023/assets/bp_10.dat}; \label{line:interm}
     \end{axis}
     \end{tikzpicture}
    }
  \end{minipage}\hfill
  \begin{minipage}[c]{0.5\textwidth}
    \caption{\textbf{Necessity of Iterative Tightening.} Our approach enables us to iteratively tighten the bounds of all layers, with each iteration allowing for a smaller approximation ratio with respect to the true preimage.
     Green: Sum of bound intervals for all neurons in the third layer (second hidden layer); Blue: ratio between volumes of over-approximation and preimage.
     Measured for step $t=10$ of the control benchmark defined in Section \ref{sec:control-results}.
     Note that the improvement from iterative tightening is two orders of magnitude for intermediate bounds, and four orders of magnitude for the volume of the over-approximation.} 
    \label{fig:iterating}
  \end{minipage}
\end{center}
\end{figure}




\begin{algorithm}[tb]
   \caption{INVPROP. Can be applied to branches for non-convex overapproximations.
   Lower and upper bound of all neurons and all constraints are optimized using distinct $\alpha$ and $\gamma$ values in $g$.
   }
   \label{alg:invprop}
\begin{algorithmic}
   \STATE Initialize $\bs{l}^{(i)}, \bs{u}^{(i)}$ via cheap bound propagation methods for the forward verification problem
   \WHILE{lower bounds for $\bs{c}^{\top}\hat{\bs{x}}^{(0)}$ are improving}
      \FOR{$i \in \{L-1, L-2, \ldots, 1, 0\}$}
      \FOR{$j \in \text{layer }i\text{ neurons}, b \in \{\text{lower, upper}\}$}
      \STATE Optimize $g^b_{ij}$ from Theorem~\ref{thm:general-gamma-crown} via gradient ascent to improve bound on neuron $j$ in layer $i$ with sense $b$ (lower/upper)
      \STATE \algorithmicif\ $b=$upper \algorithmicthen\ update $u^{(i)}_j$ \algorithmicelse\ update $l^{(i)}_j$
      \ENDFOR
      \ENDFOR
      \STATE Optimize $g_{\bs{c}}$ from Theorem~\ref{thm:general-gamma-crown} via gradient ascent to improve lower bound on $\bs{c}^{\top}\bsx$ for all $\bs{c}$
   \ENDWHILE
\end{algorithmic}
\end{algorithm}

\paragraph{Connection to forward verification}\label{sec:reduction} 
Our bound in Theorem \ref{thm:gamma-crown} introduces the dual variable $\bs{\gamma}$, which enforces the output constraint during optimization. In fact, we can use this variable to get a better conceptual interpretation of our result. For the optimization problem in \eqref{eq:advopt}, taking the dual with respect to the constraint $\mathbf{H}f\br{\bs{x}}+\mathbf{d} \leq \mathbf{0}$ yields the lower bound
$$\begin{aligned}
\max_{\bs{\gamma}} \min _{\bs{x}} \quad & \bs{c}^{\top}\bsx + \bs{\gamma}^\top\left(\mathbf{H}f\left(\bs{x}\right) + \mathbf{d}\right)\\
\text{ s.t. } \quad & \bs{x} \in \mathcal{X};\quad\bs{\gamma} \geq 0
\end{aligned}$$
The objective can be represented as minimizing a linear function of the output of a residual neural network with a skip connection from the input to the output layer, subject to constraints on the input. Now, $f(\bs{x})$ appears in the objective, similar to the standard forward verification problem with an augmented network architecture. Similarly, optimizing an intermediate bound is equivalent to a skip connection from layer $i$ to the output (Appendix \ref{sec:skip-connection}).
These connections allow us to implement our method using standard verification tools~\citep{brix2023vnncomp,xu2021fast,xu2020automatic}. 

When $f$ is a feedforward ReLU network, the lower bound described in this section is precisely the same as Theorem~\ref{thm:gamma-crown} (since both are solving the dual of the same linear program). This shows that the introduction of $\bs{\gamma}$ constitutes our generalization to the bound propagation framework.

\paragraph{Selection of $\bs{c}$}
The dependence of the optimization on $\bs{c}$ is not a major limitation. 
The bounds of all intermediate neurons are optimized only with respect to the input and output constraints, not the hyperplanes used to describe the input half-space.
Thus, their optimized bounds can be shared across the optimization of all input hyperplanes. For applications such as adversarial robustness, the box constraint of two hyperplanes per dimension could be chosen to scalably bound the input. The selection of appropriate $\bs{c}$ depends on the application and is not our main focus.

\subsection{Branch and Bound}\label{sec:branching}

Though more rounds of iterative tightening will lower the gap in $\Sover \supseteq \SLP$, our current formulation still faces two sources of looseness: the gap in $\SLP \supseteq \SMILP$ and the gap in $\SMILP \supseteq \preimage$. To overcome both of these issues, we can make use of branching~\citep{bunel2020branch}. While there are several possibilities here, we focus on input branching, which gave the biggest empirical gains in our experiments described in Section \ref{sec:results}. More concretely, we divide the input domain $\Xcal=[\bs{l}^{(0)}, \bs{u}^{(0)}]$ into several regions by choosing a coordinate $i$ and computing 
\[\Xcal_a = \Xcal \cap \{\bs{x}: \bs{x}_i \geq s_i\}, \Xcal_b = \Xcal \cap \{\bs{x}: \bs{x}_i \leq s_i\}\]
so that $\Xcal=\Xcal_a \cup \Xcal_b$. Doing this recursively, we obtain several regions and can compute an overapproximation of $\preimage$ when the input is in each of those regions, and take a union of the resulting sets. The finer the input domain of our over-approximation, the tighter the approximation.

\section{Results}\label{sec:results}

We evaluate INVPROP on three different benchmarks. Across benchmarks, we find orders of magnitude improvement over prior work, and our methods work even for relatively large networks (167k neurons) and high dimensionality inputs (8112 dimensions). 
All implementation details are described in Appendix~\ref{sec:implementation} and the utilized hardware is described in Appendix~\ref{sec:hardwar}.

\subsection{Backward Reachability Analysis for Neural Feedback Loops}
\label{sec:control-results}
\paragraph{Double Integrator} After combining the state-feedback control policy $\pi$ of the benchmark with the state transition function from Section \ref{sec:application}, (see Section \ref{sec:controlbenchmarkencoding}), we get a three layer MLP with $12$, $7$, and $2$ neurons, which is typical for applications in this setting. We model multiple time steps via composing copies of this function.\footnote{For example, the $10$ time step transition function can be represented as a $21$ layer MLP (after fusing consecutive linear layers in the composition).} Moreover, we leverage intermediate bounds for time step $t$ to bound the $t+1$ time step transition. We measure the tightness of an over-approximation using its Approximation Ratio, defined as $\frac{\text{vol}(\Sover)}{\text{vol}(\preimage)}$. Both of these volumes are heuristically estimated with 1 million samples of the input space.
While INVPROP allows the optimization of arbitrary input cutting planes, we use 40 planes with slopes of equally distributed angles.

We find that INVPROP is significantly tighter and faster than ReBreach-LP, the SOTA available method for this problem \citep{rober2022hybrid, rober2022backward}.\footnote{We do not compare with \cite{everett2022drip} which improves upon ReBreach-LP, as their implementation faces numerical instability when applied for ten timesteps. However, preliminary experiments show that we also outperform their bounds.} 
As evident in Figure \ref{fig:control10steps}, ReBreach-LP suffers from increasingly weak bounds as $t$ increases, whereas our approach is able to compute tight bounds for all $t$.
Different to their implementation, we are able to leverage the output constraint to improve the intermediate neuron bounds. 
Furthermore, we can iteratively tighten these bounds with respect to each other. Even though we optimize many more quantities, our efficient bound propagation allows us to solve the problem faster overall. 
We provide a visualization of the improvement of the approximation ratio score and intermediate bounds over time in Figure~\ref{fig:iterating}. Both our increased tightness and speed are quantified in Table \ref{control-results-table}.

\begin{figure}[t]
\begin{center}
\begin{minipage}[c]{0.45\textwidth}
    \begin{subfigure}{0.49\linewidth}
        \includegraphics[width=\linewidth]{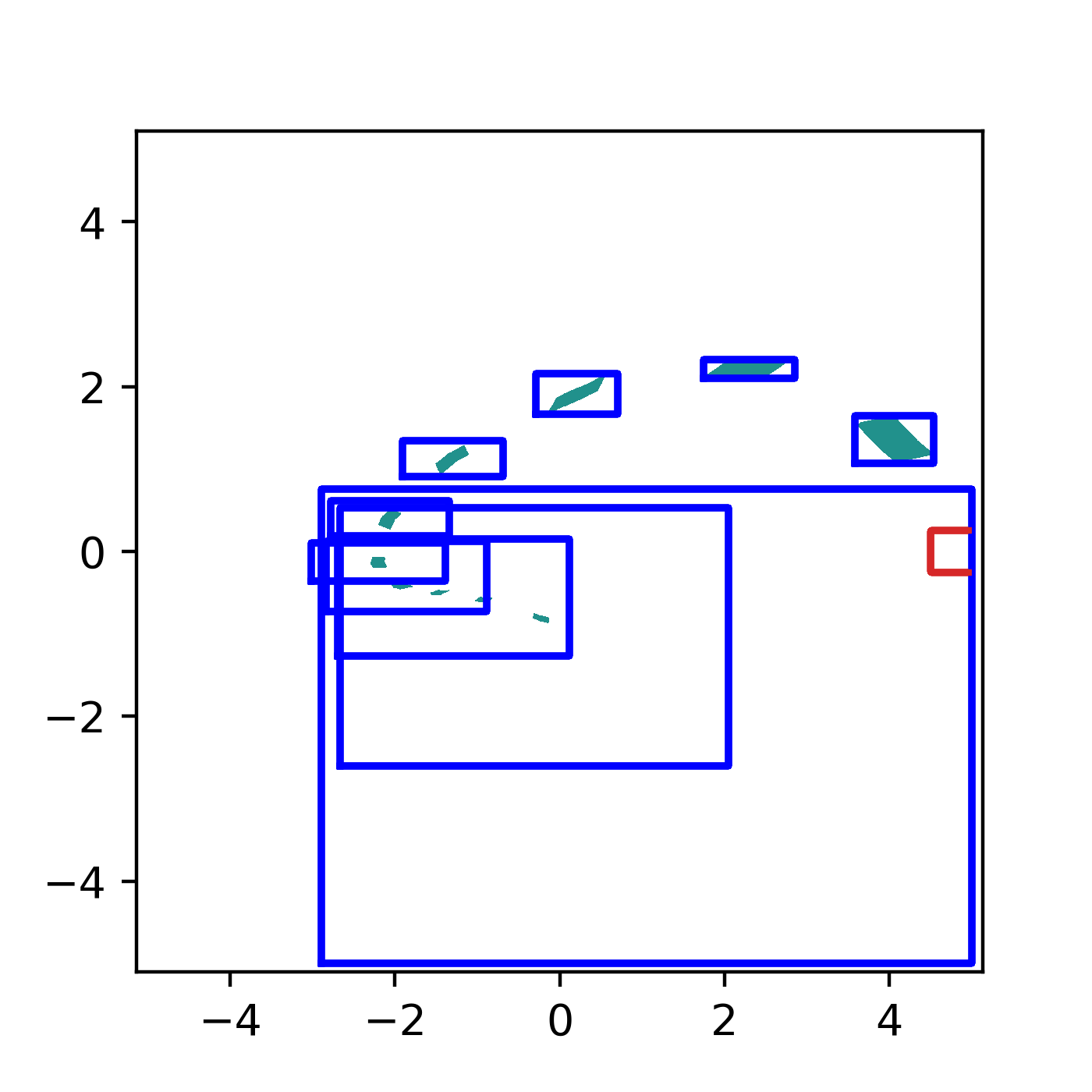}
        \caption{ReBreach-LP}
    \end{subfigure}
    \begin{subfigure}{0.49\linewidth}
        \includegraphics[width=\linewidth]{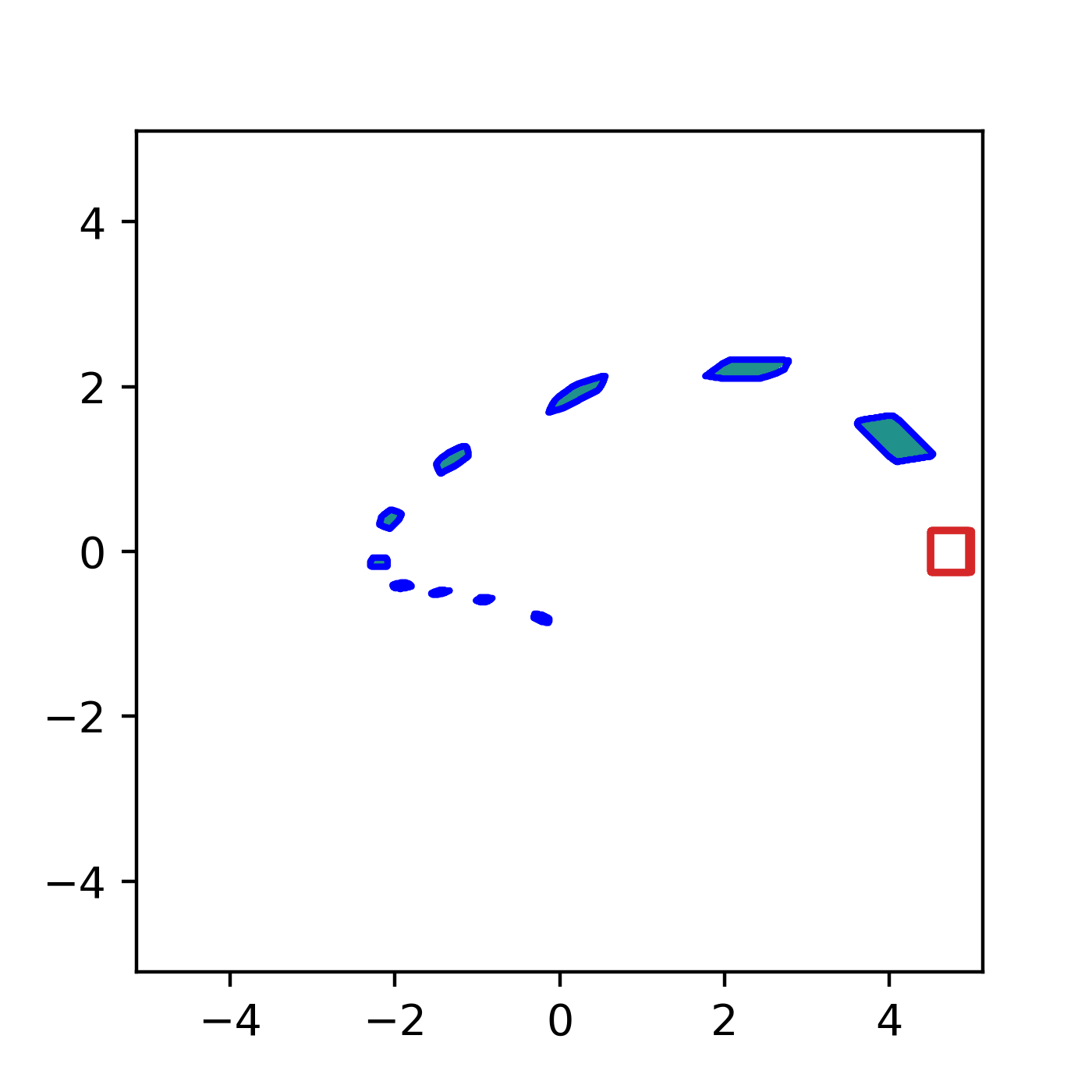}
        \caption{INVPROP}
    \end{subfigure}
\end{minipage}\hfill
\begin{minipage}[c]{0.52\textwidth}
    \caption{
        \textbf{Double Integrator.} Each green region represents the preimage of the system one to ten steps back from the obstacle (red box). Each preimage is approximated $1$ million samples. The blue bounding boxes represent over-approximations, and INVPROP tightly over-approximates preimages even ten steps back.
    } \label{fig:control10steps}
\end{minipage}
\end{center}
\vspace{-0.2cm}
\end{figure}


\begin{table}
\begin{minipage}{0.4\textwidth}
    \caption{\textbf{Double Integrator Results.} INVPROP vs SOTA (from Figure~\ref{fig:control10steps}). We are significantly tighter and faster than ReBreach-LP method even without input branching.}
    \label{control-results-table}
\end{minipage}\hfill
\begin{minipage}{0.55\textwidth}
    \begin{center}
    \begin{small}
    \begin{sc}
    \begin{tabular}{lcc}
    \toprule
    Method & ReBreach-LP & INVPROP \\
    \midrule
    Input Branching & yes & no \\
    Approx Ratio $\downarrow$  & 4021.47 & 1.46 \\
    Time (sec) & $42.86\pm0.04$ & $17.89\pm0.03$ \\
    \bottomrule
    \end{tabular}
    \end{sc}
    \end{small}
    \end{center}
\end{minipage}
\end{table}

\paragraph{6D Quadrotor} We consider over-approximating the preimage of the linearized quadrotor discussed in Figure 15 of \cite{rober2022backward}. This quadrotor has a 6-dimensional state space with 3 dimensions dedicated to position and is given by the dynamics $$\bs{x}_{t+1} = \begin{bmatrix}
1 & 0 & 0 & 1 & 0 & 0 \\
0 & 1 & 0 & 0 & 1 & 0 \\
0 & 0 & 1 & 0 & 0 & 1 \\
0 & 0 & 0 & 1 & 0 & 0 \\
0 & 0 & 0 & 0 & 1 & 0 \\
0 & 0 & 0 & 0 & 0 & 1 \\
\end{bmatrix}\bs{x}_t + \begin{bmatrix}
0.5 & 0 & 0 \\
0 & 0.5 & 0 \\
0 & 0 & 0.5 \\
1 & 0 & 0 \\
0 & 1 & 0 \\
0 & 0 & 1 \\
\end{bmatrix}\pi(\bs{x}_t)$$ 
The policy, after performing the same encoding as the double integrator and clipping policy outputs, forms a five layer MLP with 26, 26, 9, 9, and 3 neurons. For these dynamics, we want to prove that the initial state range of $[-5.25,-4.75]\times[-.25,.25]\times[2.25,2.75]\times[0.95,0.99]\times[-0.01,0.01]\times[-0.01,0.01]$ (black box in Figure \ref{fig:quadrotor}) will never collide with the obstacle at $[-1,1]\times[-1,1]\times[1.5,3.5]\times[-1,1]\times[-1,1]\times[-1,1]$ (red box). We try to find the tightest bounding box in the 3 position dimensions, reflecting a total of 6 input cutting planes (one for each direction in each dimension). To benchmark against prior work, we run ReBreach-LP with no output partitioning and ReBreach-LP with 15625 output partitions. We run our algorithm with no branching. In Figure \ref{fig:quadrotor}, we compare these three strategies, and we find that our over-approximation of the pre-image (in blue) is smaller than the state-of-the-art over-approximations by $257$ times while being $3.29$ times faster. Therefore, our algorithm scales well to higher dimensional control examples.

\begin{figure}
\centering
\begin{subfigure}{0.31\textwidth}
  \centering
  \includegraphics[width=0.9\linewidth]{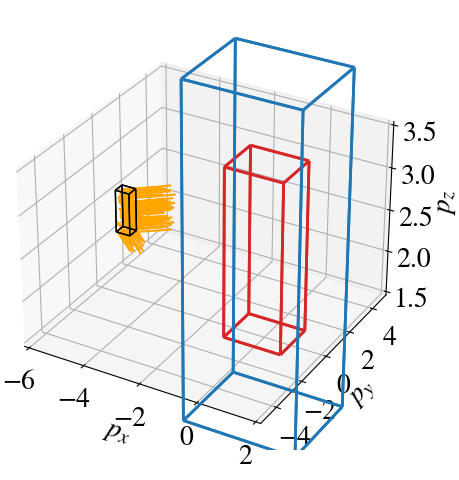}
  \caption{ReBreach-LP no partitions}
\end{subfigure}
\begin{subfigure}{0.31\textwidth}
  \centering
  \includegraphics[width=0.9\linewidth]{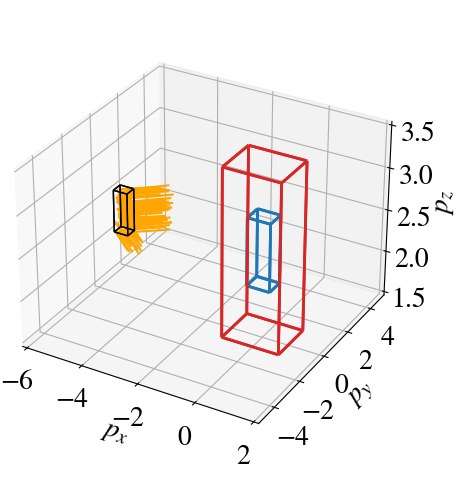}
  \caption{ReBreach-LP 15625 partitions}
\end{subfigure}
\begin{subfigure}{0.31\textwidth}
  \centering
  \includegraphics[width=0.9\linewidth]{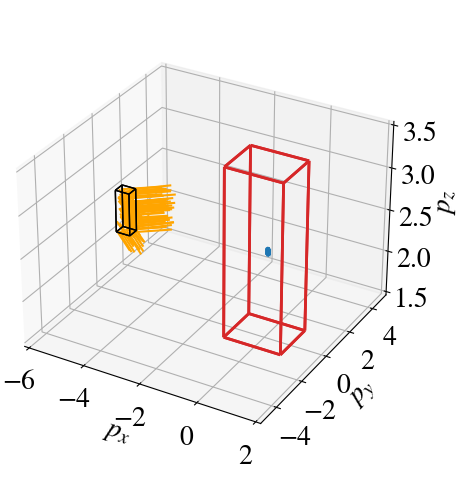}
  \caption{INVPROP}
\end{subfigure}
\caption{\textbf{6D Quadrotor.} We visualize the three dimensions of the 6D quadrotor that specify physical position. We consider the initial state range of the black box and simulate a few one-step trajectories, indicated by the orange lines. Our obstacle is the red box. For the three images, the blue box represents the over-approximation computed by (a) ReBreach-LP with no partitioning, (b) ReBreach-LP with 15625 partitions, and (c) INVPROP, respectively. A tighter blue box is better.}
\label{fig:quadrotor}
\end{figure}

\begin{table}
\begin{minipage}{0.29\textwidth}
    \caption{\textbf{Quadrotor Results.} We compare the tightness of the over-approximations and run time (mean and std taken over 5 runs).}
    \label{quadrotor-results-table}
\end{minipage}\hfill
\begin{minipage}{0.68\textwidth}
    \begin{center}
    \begin{small}
    \begin{sc}
    \begin{tabular}{lccc}
    \toprule
     Method & ReBreach-LP & ReBreach-LP & INVPROP \\
     \midrule
     Partitions & N/A & 15625 & N/A \\
     Approx Vol $\downarrow$ & 64 & 0.064 & 0.0000249 \\
     Time (sec) & 11.8 $\pm$ 2.2 & 662.2 $\pm$ 31.4 & 213.8 $\pm$ 11.1 \\
    \bottomrule
    \end{tabular}
    \end{sc}
    \end{small}
    \end{center}
\end{minipage}
\end{table}

\subsection{Robustness Verification}
\label{sec:evalnnverification}
Similar to us, \cite{pengfei2021improving} encode the implicit output constraint of the local robustness verification query to tighten the bounds on neurons.
However, they need an LP solver, which does not scale to large problems.
The state-of-the-art verification toolkit $\abcrown$
\citep{zhang2018efficient, xu2021fast, wang2021beta, zhang2022general} is able to scale to large networks, but does not currently utilize the implicit output constraint.
We demonstrate the benefit of encoding the output constraint by extending $\abcrown$ and comparing the performance on the benchmark used by \cite{pengfei2021improving} as well as benchmarks from the VNN-COMP 2023 \citep{brix2023vnncomp}.

\paragraph{MNIST}
\cite{pengfei2021improving} provides four networks with ReLU activation functions and dense linear layers of 2/6/6/9 layers of 100/100/200/200 neurons, each trained on MNIST. 
For each network, 50 inputs were tested for local robustness.
For the complete benchmark definition, we refer to \cite{pengfei2021improving}.
The DeepSRGR results are taken from \cite{pengfei2021improving}, they do not report a timeout for their experiments.
Both $\abcrown$ and $\abcrown$+INVPROP were run with a per-input timeout of 5 minutes.
Except for the input bounds, all bounds of all layers of each network are tightened using output constraints.
We report the results in Table~\ref{table:localrobustness}.
Notably, we can verify more instances than the SOTA tool $\abcrown$, in sometimes less than one fifth the average runtime.
We include a more detailed comparison between the results for $\abcrown$ and $\abcrown$+INVPROP in Appendix~\ref{sec:detailedrobustnessbenchmarks}.

\paragraph{VNN-COMP '23: YOLO}
In 2023, the VNN-COMP contained YOLO as an unscored benchmark for object detection.
It is a modified version of YOLOv2 \citep{YOLOv2} with a network of 167k neurons and 5 residual layers, and is available at \citep{YOLObenchmarkDefinition, YOLObenchmarkAsUsed}.
The network processes $3 \times 52 \times 52$ images and uses convolutions, average pooling and ReLU activations.
This benchmark is well suited for INVPROP, as the definition of the adversarial examples is a conjunction of constraints over the output.
Therefore, a strong output constraint can be used to tighten the bounds of intermediate layers.
The benchmark consists of 464 instances, of which 72 were randomly selected for the VNN-COMP.
For our comparison, we remove those 348 instances that $\abcrown$ verifies as robust without tightening the bounds of any intermediate layer.
Those instances are verified within less than four seconds each by both methods, with no room for improvement.
We compare the performance of $\abcrown$ and $\abcrown$ + INVPROP on the remaining 116 instances in Figure~\ref{fig:yolobenchmark}.
$\abcrown$ can verify 48 instances, all other instances reach the timeout of five minutes.
After extending $\abcrown$ with INVPROP on the last two intermediate layers, almost all instances can be solved faster, and 6 previously timed out instances become verifiable.


\begin{figure}
\begin{center}
\begin{minipage}[c]{0.35\textwidth}
    \includegraphics[trim={0 2.3cm 0 12cm},clip,width=\columnwidth]{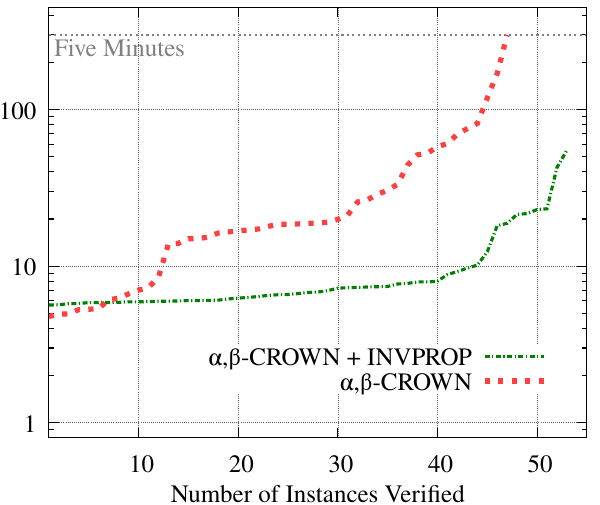}
  \end{minipage}\hfill
  \begin{minipage}[c]{0.6\textwidth}
    \caption{\textbf{YOLO Results.} Runtime comparison of $\alpha,\!\beta$-CROWN and $\alpha,\!\beta$-CROWN+INVPROP on the YOLO benchmark (167k neurons and 5 residual layers).
    For the comparison, only those test instances were used that could not immediately be verified by $\alpha,\!\beta$-CROWN without any iterative tightening of intermediate layer bounds. $\alpha,\!\beta$-CROWN+INVPROP can verify more properties and is faster for almost all instances.} \label{fig:yolobenchmark}
  \end{minipage}
\end{center}
\end{figure}

\begin{table}
\vspace{-0.8cm}
\caption{\textbf{MNIST Results.} Robustness verification with DeepSRGR (output constraints via LP solver), $\alpha,\!\beta$-CROWN (GPU support) and $\alpha,\!\beta$-CROWN+INVPROP (output constraints with GPU support)}
\label{table:localrobustness}
\vskip 0.15in
\begin{center}
\begin{small}
\begin{sc}
\begin{tabular}{lcccccc}
\toprule
Network & \multicolumn{2}{c}{DeepSRGR} & \multicolumn{2}{c}{$\alpha,\!\beta$-CROWN} & \multicolumn{2}{c}{$\alpha,\!\beta$-CROWN + INVPROP} \\
 & verif. & avg. time [s] & verif. & avg. time [s] & verif. & avg. time [s] \\ 
\midrule
FFN4 (3x100) & 35 & 781 & 45 & 7.4 & 45 & 6.4 \\
FFN5 (6x100) & 31 & 1689 & 38 & 6.7 & 39 & 8.7 \\
FFN6 (6x200) & 31 & 6178 & 29 & 5.4 & 30 & 1.0 \\
FFN7 (9x200) & 36 & 8960 & 37 & 3.2 & 40 & 1.9 \\
\bottomrule
\end{tabular}
\end{sc}
\end{small}
\end{center}
\end{table}

\subsection{OOD Detection}%
\label{sec:results-ood}%
Consider the calibrated OOD detector presented as discussed in Section \ref{sec:application}, encoded by a four layer MLP with $200$, $200$, $3$, and $2$ neurons. We over-approximate the set of inputs which induce a sufficiently high in-distribution (ID) confidence (measured by $\max\{y_0, y_1\} > y_2$) using 40 hyper-planes of equal slope, pictured in green in Figure \ref{branching-comparison-lines}.
This set is non-convex, making the convex hull a poor over-approximation. With 4 input space branches, we get a much tighter over-approximation, as shown in the right plot.
We compare the performance of our approach with and without branching over the input space with the MILP baseline (see Table \ref{ood-results-table}). This demonstrates a simple proof-of-concept for how INVPROP can be used for verifying some calibration properties.

\begin{figure}
\begin{minipage}{0.45\columnwidth}
\begin{center}
\centering
\begin{subfigure}{0.48\textwidth}
        \includegraphics[width=\linewidth]{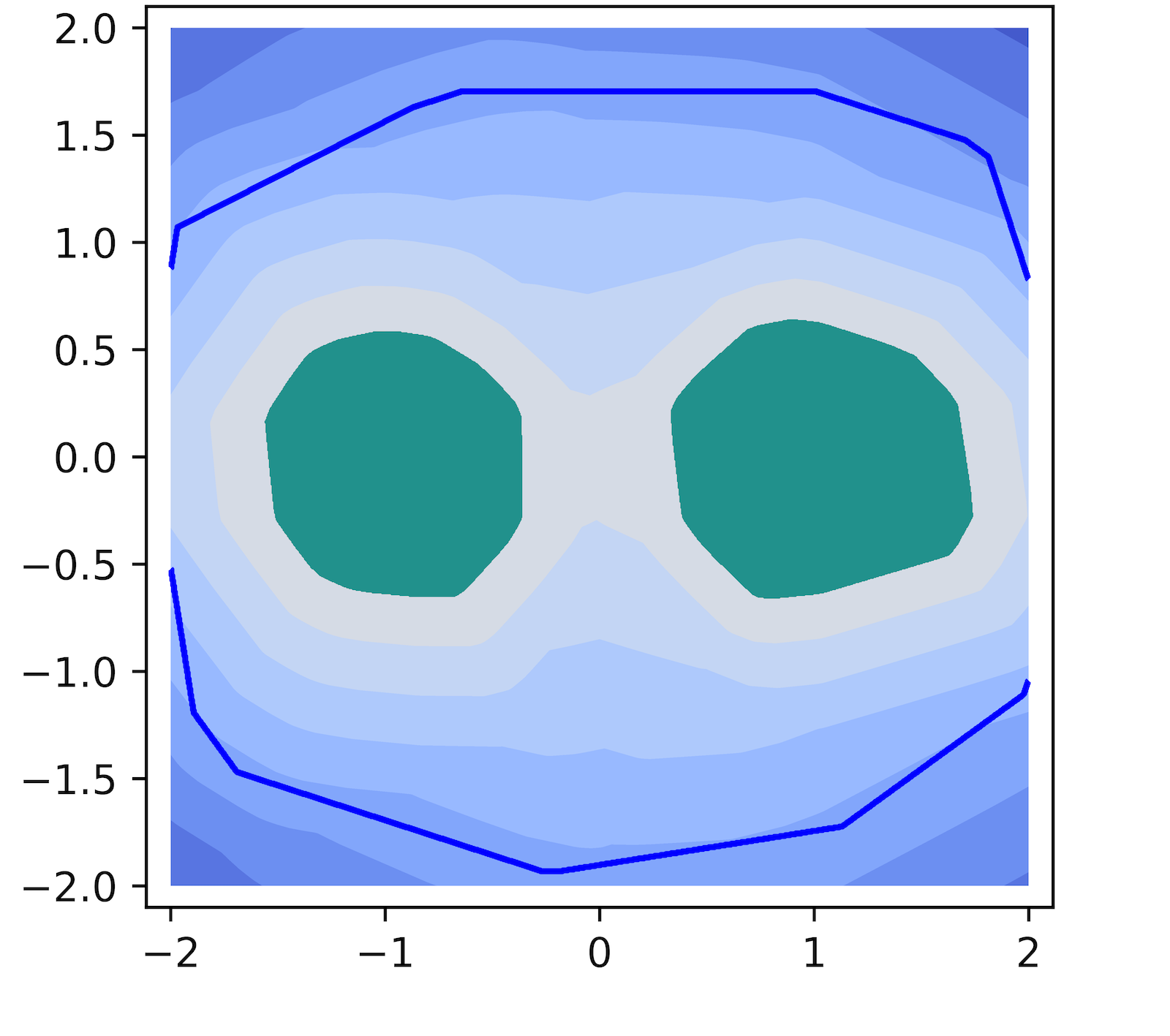}
        \caption{No branching}
    \end{subfigure}
    \begin{subfigure}{0.48\textwidth}
        \includegraphics[width=\linewidth]{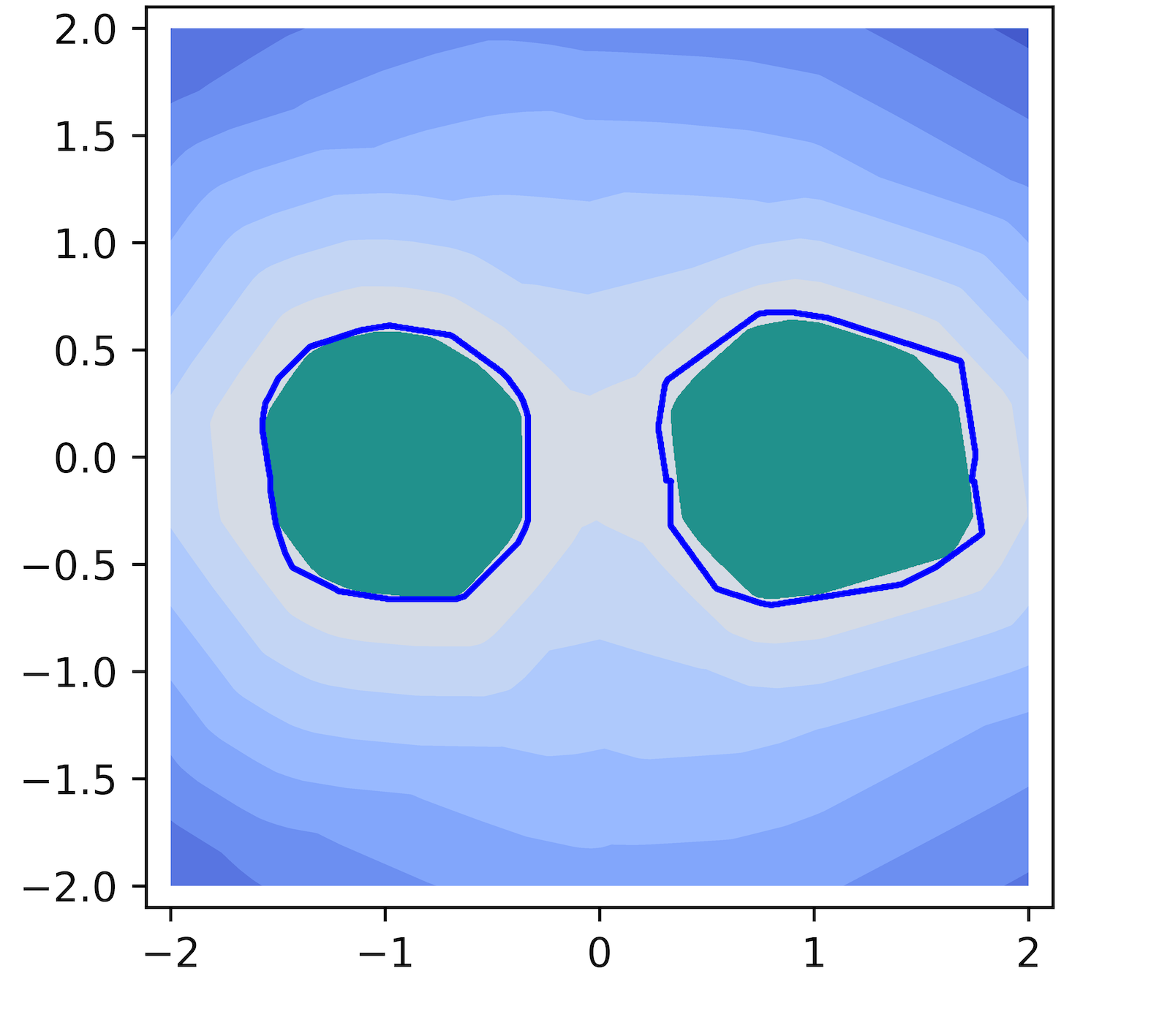}
        \caption{Branching}
    \end{subfigure}
\caption{\textbf{OOD Detection.} Green: Empirical ID inputs (1 million samples). Blue: border of the verified preimage. Branching is union of preimages for each of 4 quadrants.
}
\label{branching-comparison-lines}
\end{center}
\end{minipage}\hfill
\begin{minipage}{0.54\columnwidth}
\captionof{table}{\textbf{OOD Results.} Comparison of methods for over-approximating the preimage of the OOD region (Figure~\ref{branching-comparison-lines}). A lower approximation ratio (\textsc{Approx Rat.}) is better.}
\label{ood-results-table}
\vskip 0.15in
\begin{center}
\begin{small}
\begin{sc}
\begin{tabular}{lccc}
\toprule
Method & \makecell{Input\\Branch.} & \makecell{Approx\\Ratio} & Time (sec) \\
\midrule
MILP & no & 1.47 & 1562.26\\
INVPROP & no & 4.39 & $3.77\pm0.02$\\
INVPROP & yes & 1.14 & $12.02\pm0.06$\\
\bottomrule
\end{tabular}
\end{sc}
\end{small}
\end{center}
\end{minipage}
\end{figure}

\section{Related Work}\label{sec:related-work}

\paragraph{Formally Verified Neural Networks.}
There has been a large body of work on the formal verification of neural networks, tackling the problem from the angles of Interval Bound Propagation \citep{gowal2018effectiveness, 8418593}, Convex Relaxations \citep{wong2018provable, shiqi2018neurify, salman2019convex,dvijotham2018dual}, Abstract Interpretation \citep{NEURIPS2018_f2f44698}, LP \citep{ehlers2017formal}, SDP \citep{raghunathan2018certified}, SMT \citep{katz2017reluplex}, and MILP \citep{tjeng2017evaluating}. However, most of this work is for forward verification (i.e., bounding the NN output given a set of inputs).
Our work strictly generalizes bound propagation as presented in \citealt{xu2021fast} since setting $\bs{\gamma} = \bs{0}$ in Theorem~\ref{thm:gamma-crown} recovers its results. 

\paragraph{Formal Analysis of the Input Domain}
Prior work has studied trying to determine the inverse of a neural network \citep{KINDERMANN1990277, SAAD200778}, though these methods are prohibitively slow. \cite{zhang2018interpreting} used an LP solver to compute changes to a given input that would lead to a change in the classification result while adhering to additional criteria.
\cite{pengfei2021improving} encode the overapproximated network in an LP problem to tighten bounds by incorporating the output constraint. They require an LP solver and we compare our performance against theirs in Section~\ref{sec:evalnnverification}.
\cite{yuyi2023verification,wu2022scalable} similarly tighten intermediate bounds by using an LP solver.
\cite{dimitrov2022provably} compute input regions that only consist of adversarial examples (i.e., adhere to a given output constraint) that are maximized using an LP solver. Finally, a concurrent work \cite{zhang2023preimage} targets under- and over-approximating the preimage via LiRPA-based forward verification~\citep{xu2020automatic} with input-space and ReLU-space partitioning; since our work can be viewed as a generalization of LiPRA, their work may benefit from our results of tightening the intermediate layer bounds using output constraints during bound propagation.


\paragraph{Formal Verification for Neural Feedback Loops.}\label{sec:related-work-controls}
Our control application was motivated by the growing body of work on backward reachability analysis and over-approximating states that result in a target set \citep{everett2021reachability}. The original method solely utilized input constraints for deriving intermediate bounds. The later development of \citealt{everett2022drip} improved upon this by optimizing $\bs{l}^{(0)}$ and $\bs{u^{(0)}}$ with respect to output bounds but their implementation faces numerical instability when applied for many iterations. We also support the partitioning over the input space introduced by \citealt{rober2022hybrid}.
The work of \citealt{9561956} explores utilizing an LP with complete neuron branching to verify control safety, which can be viewed as a domain-specific implementation of our MILP formulation of the inverse verification problem.

\paragraph{Certified OOD detection.}
There is a wide variety of OOD detection systems \citep{yang2021oodsurvey,salehi2022oodsurvey}.
Commonly, they are evaluated empirically based on known OOD data \citep{haoqi2022vim,liang2018odin,chen2021atom}.
Therefore, they fail to provide verified guarantees for their effectiveness.
In fact, many OOD detection systems are susceptible to adversarial attacks \citep{sehwag2019oodattack,chen2022robustood}.
\cite{meinke2021provably} show how to verify robustness to adversarial examples around given input samples.
\cite{berrada2021make} develop a general framework for certifying properties of output distributions of neural networks given constraints on the input distribution. 
However, this work is still constrained to verifying properties of the outputs given constraints (albeit probabilistic) on the inputs, and INVPROP is able to certify arbitrary regions of the input space that lead to confident predictions.

\section{Discussion}\label{sec:discussion}

We present the challenge of over-approximating neural network preimages and provide an efficient algorithm to solve it. By doing so, we demonstrate strong performance on multiple application areas. We believe there is a large scope for future investigation and new applications of INVPROP.

\paragraph{Limitations}
For higher-dimensional instances, our method would best work for box over-approximations ($2d$ hyperplanes in $d$ dimensions) and would struggle at more complicated shapes such as spheres. We expect more complex problems will benefit from a domain-specific strategy. Moreover, iteratively refining all of the intermediate bounds utilizing INVPROP faces a quadratic dependence on network depth, as opposed to a linear dependence in traditional forward verification. To mitigate this, output constraints could only be applied to layers close to the output, as we used for robustness. Finally, our branching strategy might not scale to higher dimensions, though this trade-off is well-studied \citep{bunel2020branch, wang2021beta, depalma2023ibp}.

\paragraph{Potential Negative Social Impact}
Our work improves reliable ML through facilitating systems that provably align with practitioner expectations and we expect INVPROP to have positive societal impact. We acknowledge that our improvements may be repurposed as model attacks, though we believe the positive use cases of our technique greatly outweigh current speculation of misusage.


\paragraph{Acknowledgements} We thank Michael Everett and Nicholas Rober for helpful discussion and feedback on the paper.
Huan Zhang acknowledges the support of the Schmidt Futures AI2050 Early Career Fellowship.

\clearpage

\bibliographystyle{plainnat}
\bibliography{paper}

\begin{thebibliography}{50}
\providecommand{\natexlab}[1]{#1}
\providecommand{\url}[1]{\texttt{#1}}
\expandafter\ifx\csname urlstyle\endcsname\relax
  \providecommand{\doi}[1]{doi: #1}\else
  \providecommand{\doi}{doi: \begingroup \urlstyle{rm}\Url}\fi

\bibitem[Berrada et~al.(2021)Berrada, Dathathri, Dvijotham, Stanforth, Bunel,
  Uesato, Gowal, and Kumar]{berrada2021make}
Leonard Berrada, Sumanth Dathathri, Krishnamurthy Dvijotham, Robert Stanforth,
  Rudy~R Bunel, Jonathan Uesato, Sven Gowal, and M~Pawan Kumar.
\newblock Make sure you're unsure: A framework for verifying probabilistic
  specifications.
\newblock \emph{Advances in Neural Information Processing Systems},
  34:\penalty0 11136--11147, 2021.

\bibitem[Boyd et~al.(2004)Boyd, Boyd, and Vandenberghe]{boyd2004convex}
Stephen Boyd, Stephen~P Boyd, and Lieven Vandenberghe.
\newblock \emph{Convex optimization}.
\newblock Cambridge University Press, 2004.

\bibitem[Brix(2023)]{YOLObenchmarkAsUsed}
Christopher Brix.
\newblock Vnn-comp 2023 benchmarks.
\newblock \url{https://github.com/ChristopherBrix/vnncomp2023_benchmarks},
  2023.

\bibitem[Brix et~al.(2023)Brix, Müller, Bak, Johnson, and
  Liu]{brix2023vnncomp}
Christopher Brix, Mark~Niklas Müller, Stanley Bak, Taylor~T. Johnson, and
  Changliu Liu.
\newblock First three years of the international verification of neural
  networks competition (vnn-comp), 2023.
\newblock URL \url{https://arxiv.org/abs/2301.05815}.

\bibitem[Bunel et~al.(2020)Bunel, Mudigonda, Turkaslan, Torr, Lu, and
  Kohli]{bunel2020branch}
Rudy Bunel, P~Mudigonda, Ilker Turkaslan, P~Torr, Jingyue Lu, and Pushmeet
  Kohli.
\newblock Branch and bound for piecewise linear neural network verification.
\newblock \emph{Journal of Machine Learning Research}, 21\penalty0 (2020),
  2020.

\bibitem[Chen et~al.(2021)Chen, Li, Wu, Liang, and Jha]{chen2021atom}
Jiefeng Chen, Yixuan Li, Xi~Wu, Yingyu Liang, and Somesh Jha.
\newblock Atom: Robustifying out-of-distribution detection using outlier
  mining.
\newblock \emph{In Proceedings of European Conference on Machine Learning and
  Principles and Practice of Knowledge Discovery in Databases (ECML PKDD)},
  2021.

\bibitem[Chen et~al.(2022)Chen, Li, Wu, Liang, and Jha]{chen2022robustood}
Jiefeng Chen, Yixuan Li, Xi~Wu, Yingyu Liang, and Somesh Jha.
\newblock Robust out-of-distribution detection for neural networks.
\newblock In \emph{The AAAI-22 Workshop on Adversarial Machine Learning and
  Beyond}, 2022.
\newblock URL \url{https://openreview.net/forum?id=WMIoz7O_DPz}.

\bibitem[Dimitrov et~al.(2022)Dimitrov, Singh, Gehr, and
  Vechev]{dimitrov2022provably}
Dimitar~Iliev Dimitrov, Gagandeep Singh, Timon Gehr, and Martin Vechev.
\newblock Provably robust adversarial examples.
\newblock In \emph{International Conference on Learning Representations}, 2022.
\newblock URL \url{https://openreview.net/forum?id=UMfhoMtIaP5}.

\bibitem[Dvijotham et~al.(2018)Dvijotham, Stanforth, Gowal, Mann, and
  Kohli]{dvijotham2018dual}
Krishnamurthy Dvijotham, Robert Stanforth, Sven Gowal, Timothy~A Mann, and
  Pushmeet Kohli.
\newblock A dual approach to scalable verification of deep networks.
\newblock In \emph{UAI}, volume~1, page~3, 2018.

\bibitem[Ehlers(2017)]{ehlers2017formal}
Ruediger Ehlers.
\newblock Formal verification of piece-wise linear feed-forward neural
  networks.
\newblock In \emph{International Symposium on Automated Technology for
  Verification and Analysis}, pages 269--286. Springer, 2017.

\bibitem[Everett et~al.(2021)Everett, Habibi, Sun, and
  How]{everett2021reachability}
Michael Everett, Golnaz Habibi, Chuangchuang Sun, and Jonathan~P How.
\newblock Reachability analysis of neural feedback loops.
\newblock \emph{IEEE Access}, 9:\penalty0 163938--163953, 2021.

\bibitem[Everett et~al.(2022)Everett, Bunel, and Omidshafiei]{everett2022drip}
Michael Everett, Rudy Bunel, and Shayegan Omidshafiei.
\newblock Drip: Domain refinement iteration with polytopes for backward
  reachability analysis of neural feedback loops.
\newblock \emph{arXiv preprint arXiv:2212.04646}, 2022.

\bibitem[Gehr et~al.(2018)Gehr, Mirman, Drachsler-Cohen, Tsankov, Chaudhuri,
  and Vechev]{8418593}
Timon Gehr, Matthew Mirman, Dana Drachsler-Cohen, Petar Tsankov, Swarat
  Chaudhuri, and Martin Vechev.
\newblock Ai2: Safety and robustness certification of neural networks with
  abstract interpretation.
\newblock In \emph{2018 IEEE Symposium on Security and Privacy (SP)}, pages
  3--18, 2018.
\newblock \doi{10.1109/SP.2018.00058}.

\bibitem[Gowal et~al.(2018)Gowal, Dvijotham, Stanforth, Bunel, Qin, Uesato,
  Arandjelovic, Mann, and Kohli]{gowal2018effectiveness}
Sven Gowal, Krishnamurthy Dvijotham, Robert Stanforth, Rudy Bunel, Chongli Qin,
  Jonathan Uesato, Relja Arandjelovic, Timothy Mann, and Pushmeet Kohli.
\newblock On the effectiveness of interval bound propagation for training
  verifiably robust models.
\newblock \emph{arXiv preprint arXiv:1810.12715}, 2018.

\bibitem[Hendrycks et~al.(2018)Hendrycks, Mazeika, and
  Dietterich]{hendrycks2018deep}
Dan Hendrycks, Mantas Mazeika, and Thomas Dietterich.
\newblock Deep anomaly detection with outlier exposure.
\newblock \emph{arXiv preprint arXiv:1812.04606}, 2018.

\bibitem[Hu et~al.(2020)Hu, Fazlyab, Morari, and
  Pappas]{hu2020doubleintegrator}
Haimin Hu, Mahyar Fazlyab, Manfred Morari, and George~J. Pappas.
\newblock Reach-sdp: Reachability analysis of closed-loop systems with neural
  network controllers via semidefinite programming.
\newblock In \emph{2020 59th IEEE Conference on Decision and Control (CDC)},
  pages 5929--5934, 2020.
\newblock \doi{10.1109/CDC42340.2020.9304296}.

\bibitem[Katz et~al.(2017)Katz, Barrett, Dill, Julian, and
  Kochenderfer]{katz2017reluplex}
Guy Katz, Clark Barrett, David~L Dill, Kyle Julian, and Mykel~J Kochenderfer.
\newblock Reluplex: An efficient smt solver for verifying deep neural networks.
\newblock In \emph{International conference on computer aided verification},
  pages 97--117. Springer, 2017.

\bibitem[Kindermann and Linden(1990)]{KINDERMANN1990277}
J~Kindermann and A~Linden.
\newblock Inversion of neural networks by gradient descent.
\newblock \emph{Parallel Computing}, 14\penalty0 (3):\penalty0 277--286, 1990.
\newblock ISSN 0167-8191.
\newblock \doi{https://doi.org/10.1016/0167-8191(90)90081-J}.
\newblock URL
  \url{https://www.sciencedirect.com/science/article/pii/016781919090081J}.

\bibitem[Liang et~al.(2018)Liang, Li, and Srikant]{liang2018odin}
Shiyu Liang, Yixuan Li, and R.~Srikant.
\newblock Enhancing the reliability of out-of-distribution image detection in
  neural networks.
\newblock In \emph{International Conference on Learning Representations}, 2018.
\newblock URL \url{https://openreview.net/forum?id=H1VGkIxRZ}.

\bibitem[Meinke et~al.(2021)Meinke, Bitterwolf, and Hein]{meinke2021provably}
Alexander Meinke, Julian Bitterwolf, and Matthias Hein.
\newblock Provably robust detection of out-of-distribution data (almost) for
  free.
\newblock \emph{arXiv preprint arXiv:2106.04260}, 2021.

\bibitem[Müller et~al.(2023)Müller, Brix, Bak, Liu, and
  Johnson]{mueller2023international}
Mark~Niklas Müller, Christopher Brix, Stanley Bak, Changliu Liu, and Taylor~T.
  Johnson.
\newblock The third international verification of neural networks competition
  (vnn-comp 2022): Summary and results, 2023.

\bibitem[Palma et~al.(2023)Palma, Bunel, Dvijotham, Kumar, and
  Stanforth]{depalma2023ibp}
Alessandro~De Palma, Rudy Bunel, Krishnamurthy Dvijotham, M.~Pawan Kumar, and
  Robert Stanforth.
\newblock Ibp regularization for verified adversarial robustness via
  branch-and-bound, 2023.

\bibitem[Paszke et~al.(2019)Paszke, Gross, Massa, Lerer, Bradbury, Chanan,
  Killeen, Lin, Gimelshein, Antiga, Desmaison, Kopf, Yang, DeVito, Raison,
  Tejani, Chilamkurthy, Steiner, Fang, Bai, and Chintala]{NEURIPS2019_9015}
Adam Paszke, Sam Gross, Francisco Massa, Adam Lerer, James Bradbury, Gregory
  Chanan, Trevor Killeen, Zeming Lin, Natalia Gimelshein, Luca Antiga, Alban
  Desmaison, Andreas Kopf, Edward Yang, Zachary DeVito, Martin Raison, Alykhan
  Tejani, Sasank Chilamkurthy, Benoit Steiner, Lu~Fang, Junjie Bai, and Soumith
  Chintala.
\newblock Pytorch: An imperative style, high-performance deep learning library.
\newblock In \emph{Advances in Neural Information Processing Systems 32}, pages
  8024--8035. Curran Associates, Inc., 2019.
\newblock URL
  \url{http://papers.neurips.cc/paper/9015-pytorch-an-imperative-style-high-performance-deep-learning-library.pdf}.

\bibitem[Raghunathan et~al.(2018)Raghunathan, Steinhardt, and
  Liang]{raghunathan2018certified}
Aditi Raghunathan, Jacob Steinhardt, and Percy Liang.
\newblock Certified defenses against adversarial examples.
\newblock \emph{arXiv preprint arXiv:1801.09344}, 2018.

\bibitem[Redmon and Farhadi(2016)]{YOLOv2}
Joseph Redmon and Ali Farhadi.
\newblock {YOLO9000:} better, faster, stronger.
\newblock \emph{CoRR}, abs/1612.08242, 2016.
\newblock URL \url{http://arxiv.org/abs/1612.08242}.

\bibitem[Rober et~al.(2022{\natexlab{a}})Rober, Everett, Zhang, and
  How]{rober2022hybrid}
Nicholas Rober, Michael Everett, Songan Zhang, and Jonathan~P How.
\newblock A hybrid partitioning strategy for backward reachability of neural
  feedback loops.
\newblock \emph{arXiv preprint arXiv:2210.07918}, 2022{\natexlab{a}}.

\bibitem[Rober et~al.(2022{\natexlab{b}})Rober, Katz, Sidrane, Yel, Everett,
  Kochenderfer, and How]{rober2022backward}
Nicholas Rober, Sydney~M Katz, Chelsea Sidrane, Esen Yel, Michael Everett,
  Mykel~J Kochenderfer, and Jonathan~P How.
\newblock Backward reachability analysis of neural feedback loops: Techniques
  for linear and nonlinear systems.
\newblock \emph{arXiv preprint arXiv:2209.14076}, 2022{\natexlab{b}}.

\bibitem[Saad and Wunsch(2007)]{SAAD200778}
Emad~W. Saad and Donald~C. Wunsch.
\newblock Neural network explanation using inversion.
\newblock \emph{Neural Networks}, 20\penalty0 (1):\penalty0 78--93, 2007.
\newblock ISSN 0893-6080.
\newblock \doi{https://doi.org/10.1016/j.neunet.2006.07.005}.
\newblock URL
  \url{https://www.sciencedirect.com/science/article/pii/S0893608006001730}.

\bibitem[Salehi et~al.(2022)Salehi, Mirzaei, Hendrycks, Li, Rohban, and
  Sabokrou]{salehi2022oodsurvey}
Mohammadreza Salehi, Hossein Mirzaei, Dan Hendrycks, Yixuan Li,
  Mohammad~Hossein Rohban, and Mohammad Sabokrou.
\newblock A unified survey on anomaly, novelty, open-set, and out
  of-distribution detection: Solutions and future challenges.
\newblock \emph{Transactions on Machine Learning Research}, 2022.
\newblock URL \url{https://openreview.net/forum?id=aRtjVZvbpK}.

\bibitem[Salman et~al.(2019)Salman, Yang, Zhang, Hsieh, and
  Zhang]{salman2019convex}
Hadi Salman, Greg Yang, Huan Zhang, Cho-Jui Hsieh, and Pengchuan Zhang.
\newblock A convex relaxation barrier to tight robustness verification of
  neural networks.
\newblock \emph{Advances in Neural Information Processing Systems}, 32, 2019.

\bibitem[Sehwag et~al.(2019)Sehwag, Bhagoji, Song, Sitawarin, Cullina, Chiang,
  and Mittal]{sehwag2019oodattack}
Vikash Sehwag, Arjun~Nitin Bhagoji, Liwei Song, Chawin Sitawarin, Daniel
  Cullina, Mung Chiang, and Prateek Mittal.
\newblock Analyzing the robustness of open-world machine learning.
\newblock In \emph{Proceedings of the 12th ACM Workshop on Artificial
  Intelligence and Security}, AISec'19, page 105–116, New York, NY, USA,
  2019. Association for Computing Machinery.
\newblock ISBN 9781450368339.
\newblock \doi{10.1145/3338501.3357372}.
\newblock URL \url{https://doi.org/10.1145/3338501.3357372}.

\bibitem[Shiqi et~al.(2018)Shiqi, Pei, Justin, Yang, and
  Jana]{shiqi2018neurify}
Wang Shiqi, Kexin Pei, Whitehouse Justin, Junfeng Yang, and Suman Jana.
\newblock Efficient formal safety analysis of neural networks.
\newblock In \emph{32nd Conference on Neural Information Processing Systems
  (NIPS)}, Montreal, Canada, 2018.

\bibitem[Singh et~al.(2018)Singh, Gehr, Mirman, P\"{u}schel, and
  Vechev]{NEURIPS2018_f2f44698}
Gagandeep Singh, Timon Gehr, Matthew Mirman, Markus P\"{u}schel, and Martin
  Vechev.
\newblock Fast and effective robustness certification.
\newblock In S.~Bengio, H.~Wallach, H.~Larochelle, K.~Grauman, N.~Cesa-Bianchi,
  and R.~Garnett, editors, \emph{Advances in Neural Information Processing
  Systems}, volume~31. Curran Associates, Inc., 2018.
\newblock URL
  \url{https://proceedings.neurips.cc/paper/2018/file/f2f446980d8e971ef3da97af089481c3-Paper.pdf}.

\bibitem[Singh et~al.(2019)Singh, Gehr, P\"{u}schel, and Vechev]{singh2019rsip}
Gagandeep Singh, Timon Gehr, Markus P\"{u}schel, and Martin Vechev.
\newblock An abstract domain for certifying neural networks.
\newblock \emph{Proc. ACM Program. Lang.}, 3\penalty0 (POPL), jan 2019.
\newblock \doi{10.1145/3290354}.
\newblock URL \url{https://doi.org/10.1145/3290354}.

\bibitem[Tjeng et~al.(2017)Tjeng, Xiao, and Tedrake]{tjeng2017evaluating}
Vincent Tjeng, Kai Xiao, and Russ Tedrake.
\newblock Evaluating robustness of neural networks with mixed integer
  programming.
\newblock \emph{arXiv preprint arXiv:1711.07356}, 2017.

\bibitem[Vincent and Schwager(2021)]{9561956}
Joseph~A. Vincent and Mac Schwager.
\newblock Reachable polyhedral marching (rpm): A safety verification algorithm
  for robotic systems with deep neural network components.
\newblock In \emph{2021 IEEE International Conference on Robotics and
  Automation (ICRA)}, pages 9029--9035, 2021.
\newblock \doi{10.1109/ICRA48506.2021.9561956}.

\bibitem[Wang et~al.(2022)Wang, Li, Feng, and Zhang]{haoqi2022vim}
Haoqi Wang, Zhizhong Li, Litong Feng, and Wayne Zhang.
\newblock Vim: Out-of-distribution with virtual-logit matching.
\newblock In \emph{Proceedings of the IEEE/CVF Conference on Computer Vision
  and Pattern Recognition}, 2022.

\bibitem[Wang et~al.(2021)Wang, Zhang, Xu, Lin, Jana, Hsieh, and
  Kolter]{wang2021beta}
Shiqi Wang, Huan Zhang, Kaidi Xu, Xue Lin, Suman Jana, Cho-Jui Hsieh, and
  J~Zico Kolter.
\newblock {Beta-CROWN}: Efficient bound propagation with per-neuron split
  constraints for complete and incomplete neural network verification.
\newblock \emph{Advances in Neural Information Processing Systems}, 34, 2021.

\bibitem[Wong and Kolter(2018)]{wong2018provable}
Eric Wong and Zico Kolter.
\newblock Provable defenses against adversarial examples via the convex outer
  adversarial polytope.
\newblock In \emph{International Conference on Machine Learning}, pages
  5286--5295. PMLR, 2018.

\bibitem[Wu et~al.(2022)Wu, Barrett, Sharif, Narodytska, and
  Singh]{wu2022scalable}
Haoze Wu, Clark Barrett, Mahmood Sharif, Nina Narodytska, and Gagandeep Singh.
\newblock Scalable verification of gnn-based job schedulers.
\newblock \emph{Proc. ACM Program. Lang.}, 6\penalty0 (OOPSLA2), oct 2022.
\newblock \doi{10.1145/3563325}.
\newblock URL \url{https://doi.org/10.1145/3563325}.

\bibitem[Xu et~al.(2020)Xu, Shi, Zhang, Wang, Chang, Huang, Kailkhura, Lin, and
  Hsieh]{xu2020automatic}
Kaidi Xu, Zhouxing Shi, Huan Zhang, Yihan Wang, Kai-Wei Chang, Minlie Huang,
  Bhavya Kailkhura, Xue Lin, and Cho-Jui Hsieh.
\newblock Automatic perturbation analysis for scalable certified robustness and
  beyond.
\newblock \emph{Advances in Neural Information Processing Systems}, 33, 2020.

\bibitem[Xu et~al.(2021)Xu, Zhang, Wang, Wang, Jana, Lin, and
  Hsieh]{xu2021fast}
Kaidi Xu, Huan Zhang, Shiqi Wang, Yihan Wang, Suman Jana, Xue Lin, and Cho-Jui
  Hsieh.
\newblock {Fast and Complete}: Enabling complete neural network verification
  with rapid and massively parallel incomplete verifiers.
\newblock In \emph{International Conference on Learning Representations}, 2021.
\newblock URL \url{https://openreview.net/forum?id=nVZtXBI6LNn}.

\bibitem[Yang et~al.(2021{\natexlab{a}})Yang, Zhou, Li, and
  Liu]{yang2021oodsurvey}
Jingkang Yang, Kaiyang Zhou, Yixuan Li, and Ziwei Liu.
\newblock Generalized out-of-distribution detection: A survey.
\newblock \emph{arXiv preprint arXiv:2110.11334}, 2021{\natexlab{a}}.

\bibitem[Yang et~al.(2021{\natexlab{b}})Yang, Li, Li, Huang, Wang, Sun, Xue,
  and Zhang]{pengfei2021improving}
Pengfei Yang, Renjue Li, Jianlin Li, Cheng-Chao Huang, Jingyi Wang, Jun Sun,
  Bai Xue, and Lijun Zhang.
\newblock Improving neural network verification through spurious region guided
  refinement.
\newblock In Jan~Friso Groote and Kim~Guldstrand Larsen, editors, \emph{Tools
  and Algorithms for the Construction and Analysis of Systems}, pages 389--408,
  Cham, 2021{\natexlab{b}}. Springer International Publishing.
\newblock ISBN 978-3-030-72016-2.

\bibitem[Zhang et~al.(2018{\natexlab{a}})Zhang, Weng, Chen, Hsieh, and
  Daniel]{zhang2018efficient}
Huan Zhang, Tsui-Wei Weng, Pin-Yu Chen, Cho-Jui Hsieh, and Luca Daniel.
\newblock Efficient neural network robustness certification with general
  activation functions.
\newblock \emph{Advances in Neural Information Processing Systems},
  31:\penalty0 4939--4948, 2018{\natexlab{a}}.
\newblock URL \url{https://arxiv.org/pdf/1811.00866.pdf}.

\bibitem[Zhang et~al.(2022)Zhang, Wang, Xu, Li, Li, Jana, Hsieh, and
  Kolter]{zhang2022general}
Huan Zhang, Shiqi Wang, Kaidi Xu, Linyi Li, Bo~Li, Suman Jana, Cho-Jui Hsieh,
  and J~Zico Kolter.
\newblock General cutting planes for bound-propagation-based neural network
  verification.
\newblock \emph{Advances in Neural Information Processing Systems}, 2022.

\bibitem[Zhang et~al.(2018{\natexlab{b}})Zhang, Solar-Lezama, and
  Singh]{zhang2018interpreting}
Xin Zhang, Armando Solar-Lezama, and Rishabh Singh.
\newblock Interpreting neural network judgments via minimal, stable, and
  symbolic corrections.
\newblock In S.~Bengio, H.~Wallach, H.~Larochelle, K.~Grauman, N.~Cesa-Bianchi,
  and R.~Garnett, editors, \emph{Advances in Neural Information Processing
  Systems}, volume~31. Curran Associates, Inc., 2018{\natexlab{b}}.
\newblock URL
  \url{https://proceedings.neurips.cc/paper_files/paper/2018/file/300891a62162b960cf02ce3827bb363c-Paper.pdf}.

\bibitem[Zhang et~al.(2023)Zhang, Wang, and Kwiatkowska]{zhang2023preimage}
Xiyue Zhang, Benjie Wang, and Marta Kwiatkowska.
\newblock Provable preimage under-approximation for neural networks.
\newblock \emph{arXiv preprint arXiv:2305.03686}, 2023.

\bibitem[Zhong(2023)]{YOLObenchmarkDefinition}
Xiangru Zhong.
\newblock Yolo-benchmark.
\newblock \url{https://github.com/xiangruzh/Yolo-Benchmark}, 2023.

\bibitem[Zhong et~al.(2023)Zhong, Ta, and Khoo]{yuyi2023verification}
Yuyi Zhong, Quang-Trung Ta, and Siau-Cheng Khoo.
\newblock Arena: Enhancing abstract refinement for neural network
  verification.
\newblock In Cezara Dragoi, Michael Emmi, and Jingbo Wang, editors,
  \emph{Verification, Model Checking, and Abstract Interpretation}, pages
  366--388, Cham, 2023. Springer Nature Switzerland.
\newblock ISBN 978-3-031-24950-1.

\end{thebibliography}

\clearpage
\appendix

\section{OOD Detection}
\label{sec:ooddetectionfigure}
\begin{figure}[ht]
\begin{center}
\centerline{\includegraphics[width=0.321\columnwidth]{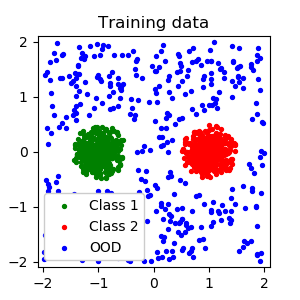} \includegraphics[width=0.429\columnwidth]{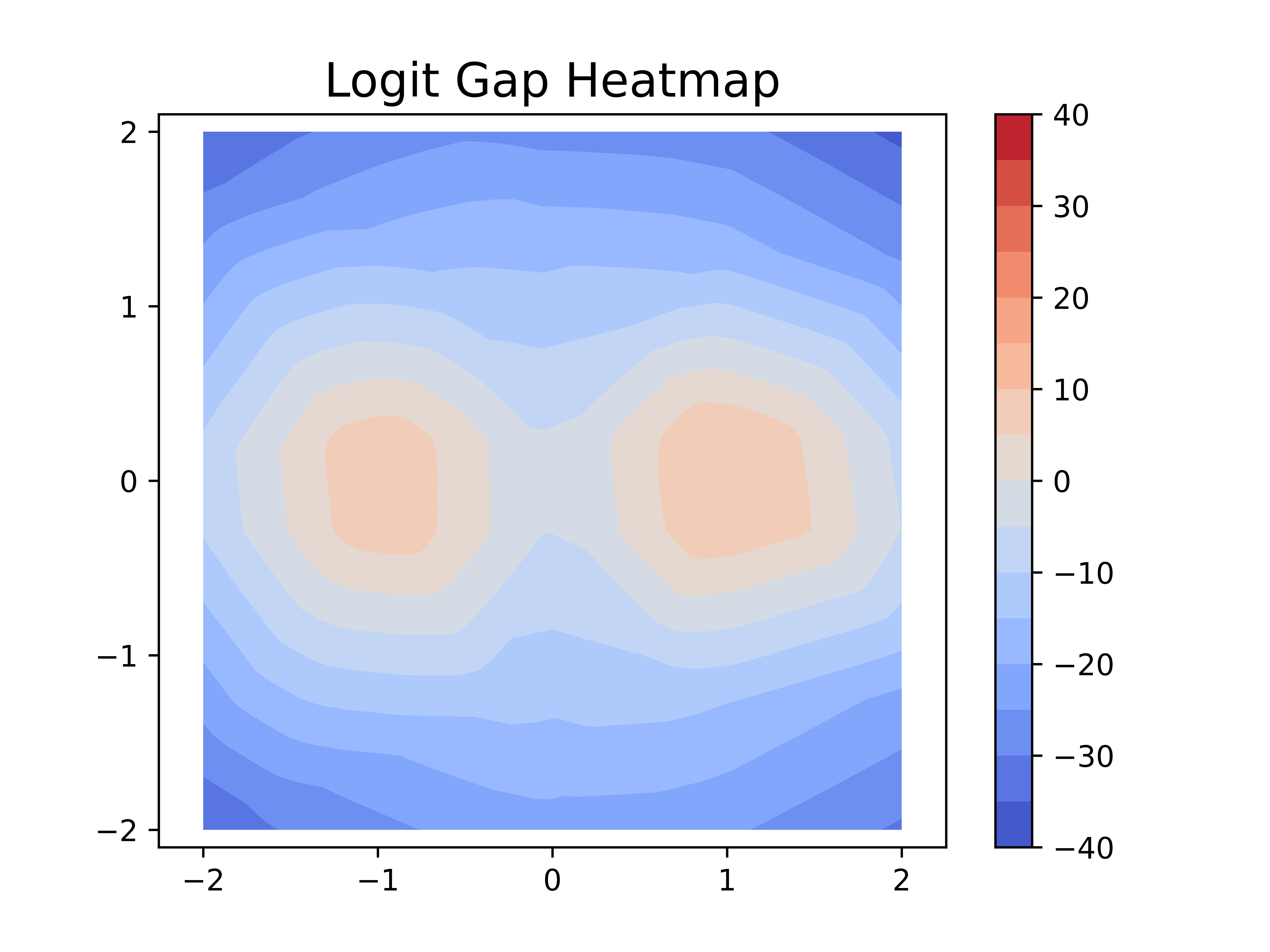}}
\caption{To improve our model's ability to detect data outside the training distribution (of green and red points), we randomly sample points that are far from every training point. In the right contour, we see that the model's confidence of in-distribution increases as we approach the centers of the distribution, demonstrating our model is more calibrated.}
\label{calibrated-toy-ood}
\end{center}
\end{figure}

For the OOD benchmark, we train a neural network on two clusters of data. Standard training would lead to an uncalibrated predictor, we add in OOD data through sampling the input space and rejecting inputs that are too close to any training data point. Training a classifier on this dataset induces the contour on the right, which is a calibrated classifier we can certify properties over. For our experiments, we consider a high-confidence output as $90\%$ confidence of the data being in-distribution.

\section{Example}
\label{intermediate-example}
\begin{exmp}
\begin{figure}[t]
\centering
\includegraphics[scale=0.5]{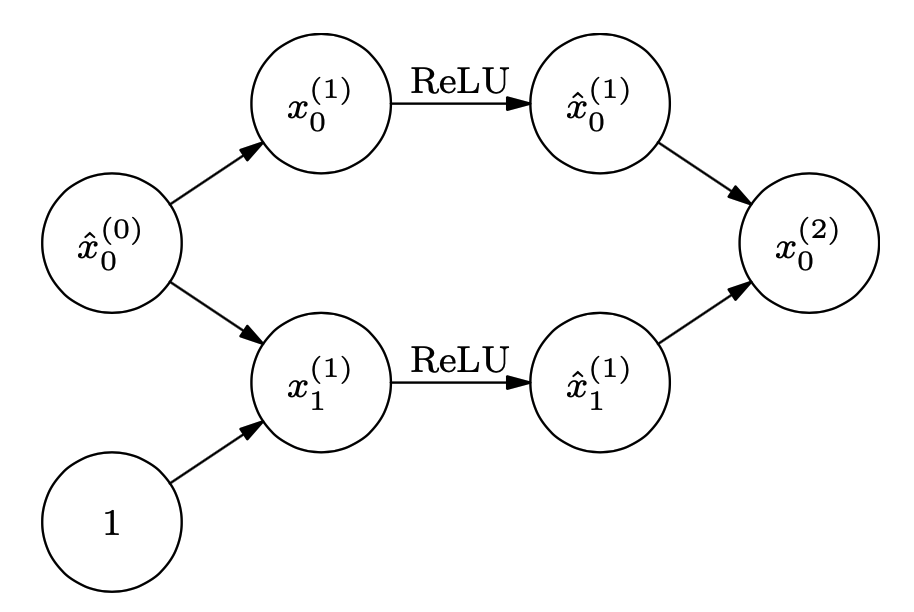}
\caption{A simple neural network from $\mathbb{R}$ to $\mathbb{R}$. Every node is the sum of its incoming nodes unless the edge is labeled ReLU.}
\label{fig:intermediate-demo}
\end{figure}

Consider the toy neural network in Figure~\ref{fig:intermediate-demo} under $\mathcal{X} = [-2, 2]$ and $\Sout = \{y : 1 \leq y \leq 1.02\}$. Suppose we wanted to find the tightest $[l^{(1)}_0, u^{(1)}_0]$ (bounds on $x^{(1)}_0$). If we only enforce the constraints preceding the ReLU (standard bound propagation), we get $[-2, 2]$. If we only enforce the constraints following the ReLU, we get $(-\infty, 1.02]$. However, when we utilize all of the constraints, we find that the true intermediate bounds are $[0, 0.01]$, which is over $100$ times tighter than the intersection of the two previous methods (all derivations provided below). Therefore, new techniques are necessary for optimizing intermediate bounds in this setting.
\end{exmp}

\subsection{Deriving Intermediate Bounds for Example \ref{intermediate-example}}\label{intermediate-example-derivation}

\paragraph{Preceding Constraints.} The linear layer gives us $\bsx^{(1)}_0 = \hat{\bsx}^{(0)}_0$. Since the only constraint we have is $\hat{\bsx}^{(0)}_0 \in [-2, 2]$, the tightest intermediate bounds we can derive are $[-2, 2]$.

\paragraph{Following Constraints.} Since the output of a ReLU is non-negative, we know that $\hat{\bsx}^{(1)}_0, \hat{\bsx}^{(1)}_1 \geq 0$. Since $\bsx^{(2)}_0 = \hat{\bsx}^{(1)}_0 +\hat{\bsx}^{(1)}_1$ and the output constraint enforces $\bsx^{(2)}_0 \leq 1.02$, we derive that $\hat{\bsx}^{(1)}_0 \leq 1.02$. If we set $\hat{\bsx}^{(1)}_0 = 1.02 - \hat{\bsx}^{(1)}_1$, we see that $\hat{\bsx}^{(1)}_0$ can take the entire interval $[0, 1.02]$. Since the only constraint we have on $\bsx^{(1)}_0$ is that $\ReLU(\hat{\bsx}^{(1)}_0) = \bsx^{(1)}$, our desired interval is the preimage of $[0, 1.02]$. This means that $\bsx^{(1)}_0$ can take any value in the interval $(-\infty, 1.02]$.

\paragraph{All Constraints.} We first note that by the first linear layer, we have $\bsx^{(1)}_1 = \bsx^{(1)}_0 + 1$. Therefore, if $\bsx^{(1)}_0$ is less than $0$, then $\bsx^{(1)}_1$ is less than $1$, which means $\ReLU(\bsx^{(1)}_0) + \ReLU(\bsx^{(1)}_1)$ is less than $1$, which contradicts the output constraint. If $\bsx^{(1)}_0$ is always non-negative, then we have that the output is equivalent to $\ReLU(\bsx^{(1)}_0) + \ReLU(\bsx^{(1)}_0 + 1) = 2\bsx^{(1)}_0 + 1$. Therefore, the output constraint implies $\bsx^{(1)}_0 \leq 0.01$. Any $\bsx^{(1)}_0 \in [0, 0.01]$ is achievable by setting the input to the desired value.

\pagebreak


\section{Optimizing Intermediate Bounds}\label{sec:intermediate-bounds-theorem}

We present a generalized theorem which provides a lower bound for any linear combination of $\hat{\bsx}^{(0)}$ and $\bsx^{(i)}$. We prove this theorem in Appendix \ref{sec:theorem-proof}. We note that by selecting the coefficients of these variables (named $\bs{c}^{(i)}$ for $i \in \{0, 1, \ldots, L-1\}$), we recover the two following functionalities which we name $g_{\bs{c}}(\bs{\alpha}, \bs{\gamma})$ and $g_{i,j}^{b}(\bs{\alpha}, \bs{\gamma})$ (respectively).

\begin{itemize}
    \item When the only nonzero coefficients are $\bs{c}^{(0)} = \bs{c}$, we recover Theorem \ref{thm:gamma-crown}. We refer to this bound as $g_{\bs{c}}(\bs{\alpha}, \bs{\gamma})$
    \item When the only nonzero coefficient is $c^{(i)}_j = 1$, we can lower bound $l^{(i)}_j$ by $g_{i,j}^{\text{lower}}(\bs{\alpha}, \bs{\gamma})$. If we set $c^{(i)}_j = -1$, we can upper bound $u^{(i)}_j$ by $-g_{i,j}^{\text{upper}}(\bs{\alpha}, \bs{\gamma})$.
\end{itemize}

With this, we present our generalized theorem.

\begin{theorem}[Lower-bounding combination of neurons]\label{thm:general-gamma-crown}

Given an output set $\Sout = \{\bs{y} : \mathbf{H}\bs{y} + \mathbf{d} \leq \mathbf{0}\}$ and vector $\bs{c}$, $g(\bs{\alpha}, \bs{\gamma})$ is a lower bound to the linear program in \eqref{eq:LP} with objective  $\bs{c}^{(0)\top}\hat{\bs{x}}^{(0)} + \sum_{i=1}^{L-1}\bs{c}^{(i)\top}\bs{x}^{(i)}$ for $\bs{0} \leq \bs{\alpha} \leq \bs{1}$, $\bs{\gamma} \geq \bs{0}$, and $g$ defined via 
$$
\begin{aligned}
g(\bs{\alpha}, \bs{\gamma}) &= \left[\bs{c}^{(0)\top} - \boldsymbol{\nu}^{(1) \top} \mathbf{W}^{(1)}\right]_+\bs{l}^{(0)}-\left[\bs{c}^{(0)\top} -\boldsymbol{\nu}^{(1) \top} \mathbf{W}^{(1)}\right]_-\bs{u}^{(0)} \\ &-\sum_{i=1}^L \boldsymbol{\nu}^{(i) \top} \mathbf{b}^{(i)} + \sum_{i=1}^{L-1} \sum_{j \in \mathcal{I}^{\pm(i)}}\left[\frac{u^{(i)}_{j}l^{(i)}_j [\hat{\nu}^{(i)}_j]_{+}}{u^{(j)}_{i} - l^{(j)}_i} \right]
\end{aligned}
$$
where every term can be directly recursively computed via
$$
\begin{aligned}
\mathcal{I}^{-(i)} &= \{j : u^{(i)}_j \leq 0\}\\
\mathcal{I}^{+(i)} &= \{j : l^{(i)}_j \geq 0\}\\
\mathcal{I}^{\pm(i)} &= \{j : l^{(i)}_j < 0 < u^{(i)}_j\}\\
\boldsymbol{\nu}^{(L)} &= -\boldsymbol{\gamma} \\
\hat{\nu}_j^{(i)} &= \boldsymbol{\nu}^{(i+1) \top} \mathbf{W}_{:, j}^{(i+1)} \\
\nu_j^{(i)} &= \left\{\begin{array}{lr}
\hat{\nu}_j^{(i)} - c^{(i)}_j, &j \in \mathcal{I}^{+(i)} \\
- c^{(i)}_j, & j \in \mathcal{I}^{-(i)} \\
\frac{u^{(i)}_{j}}{u^{(j)}_{i} - l^{(j)}_i}[\hat{\nu}^{(i)}_j]_{+} - \alpha^{(i)}_j[\hat{\nu}^{(i)}_{j}]_{-} - c^{(i)}_j, & j \in \mathcal{I}^{\pm(i)}
\end{array}\right.
\end{aligned}
$$
\end{theorem}

\section{Dual Derivation}\label{sec:theorem-proof}

We first incorporate the output constraint $\mathbf{H}\bs{x}^{(L)} + \mathbf{d} \leq \mathbf{0}$ by folding in this linear transformation of the output into the final linear layer of the network, as customary in prior work \citep{wang2021beta, zhang2022general}.

We now prove the generalized theorem as described in \ref{sec:intermediate-bounds-theorem}. As a conceptual overview of our proof, we take the Lagrange dual of the linear program to derive an unconstrained optimization problem. From here, we derive constraints that must be satisfied in the $\max \min$ formulation. These constraints yield the bound propagation procedure we display in Theorem \ref{thm:gamma-crown}.

We start with the convex relaxation. $$\begin{aligned}
\min _{\bsx, \hat{\bsx}} \quad &  \bs{c}^{(0)\top}\hat{\bsx}^{(0)} + \sum_{i=1}^{L-1}\bs{c}^{(i)\top}\bsx^{(i)} \\
\text { s.t. } \quad &  \bs{l}^{(0)} \leq \hat{\bsx}^{(0)} \leq \bs{u}^{(0)} \\
& \bsx^{(L)} \leq \boldsymbol{0} ; \\
& \bsx^{(i)} =\mathbf{W}^{(i)} \hat{\bsx}^{(i-1)}+\mathbf{b}^{(i)} ; \quad i \in[L], \\
& \hat{x}_j^{(i)} \geq 0 ; j \in \mathcal{I}^{\pm(i)} \\
& \hat{x}_j^{(i)} \geq x_j^{(i)} ; j \in \mathcal{I}^{\pm(i)} \\
& (u_j^{(i)} - l_j^{(i)})\hat{x}_j^{(i)} \leq u_j^{(i)} x_j^{(i)} - u_j^{(i)}l_j^{(i)} ; j \in \mathcal{I}^{\pm(i)} \\
& \hat{x}_j^{(i)} =x_j^{(i)} ; j \in \mathcal{I}^{+(i)} \\
& \hat{x}_j^{(i)} =0 ; j \in \mathcal{I}^{-(i)} \\
\end{aligned}$$

From this, we take the Lagrange dual of most of the constraints. In doing so, we introduce a new dual variable for each constraint. Note that we do not dualize every constraint since they can easily be dealt with later.

$$\begin{aligned}
\min_{\bsx, \hat{\bsx}} \max_{\boldsymbol{\nu}, \boldsymbol{\mu}, \boldsymbol{\tau}, \boldsymbol{\gamma}, \boldsymbol{\lambda}} \quad & \bs{c}^{(0)\top}\hat{\bsx}^{(0)} + \sum_{i=1}^{L-1}\bs{c}^{(i)\top}\bsx^{(i)} + \boldsymbol{\gamma}^{\top}\bsx^{(L)}+\sum_{i=1}^L \boldsymbol{\nu}^{(i) \top}\left(\bsx^{(i)}-\mathbf{W}^{(i)} \hat{\bsx}^{(i-1)}-\mathbf{b}^{(i)}\right) \\
& +\sum_{i=1}^{L-1} \sum_{j \in \mathcal{I}^{\pm(i)}}\left[\mu_j^{(i)}\left(-\hat{x}_j^{(i)}\right)+\tau_j^{(i)}\left(x_j^{(i)}-\hat{x}_j^{(i)}\right)+\lambda_j^{(i)}\left((u_j^{(i)} - l_j^{(i)})\hat{x}_j^{(i)} - u_j^{(i)} x_j^{(i)} + u_j^{(i)}l_j^{(i)}\right)\right] \\
\text { s.t. }  \bs{l}^{(0)} & \leq \hat{\bsx}^{(0)} \leq \bs{u}^{(0)} ; \quad \hat{x}_j^{(i)} = 0, j \in \mathcal{I}^{-(i)} ; \quad \hat{x}_j^{(i)}=x_j^{(i)}, j \in \mathcal{I}^{+(i)} \\
\boldsymbol{\mu} & \geq 0 ; \quad \boldsymbol{\tau} \geq 0 ; \quad \boldsymbol{\gamma} \geq 0 ; \quad \boldsymbol{\lambda} \geq 0
\end{aligned}$$

Since we took the dual of a linear program, the solution to the min-max optimization is equivalent to the solution of the max-min optimization by strong duality. Therefore, we can solve the following program (equivalent up to rearrangement).

$$\begin{aligned}
\max_{\boldsymbol{\nu}, \boldsymbol{\mu}, \boldsymbol{\tau}, \boldsymbol{\gamma}, \boldsymbol{\lambda}} \min_{\bsx, \hat{\bsx}} \quad &  \left(\boldsymbol{\nu}^{(L)}+\boldsymbol{\gamma}\right)^{\top} \bsx^{(L)} + \left(\bs{c}^{(0)\top} - \boldsymbol{\nu}^{(1) \top} \mathbf{W}^{(1)}\right)\hat{\bsx}^{(0)} \\
& +\sum_{i=1}^{L-1} \sum_{j \in \mathcal{I}^{+(i)}}\left(\nu_j^{(i)}-\boldsymbol{\nu}^{(i+1) \top} \mathbf{W}_{:, j}^{(i+1)} + c^{(i)}_j\right) x_j^{(i)} +\sum_{i=1}^{L-1} \sum_{j \in \mathcal{I}^{-(i)}}\left(\nu_j^{(i)} + c^{(i)}_j\right) x_j^{(i)} \\
& +\sum_{i=1}^{L-1} \sum_{j \in \mathcal{I}^{\pm(i)}}\left[\left(\nu_j^{(i)}+\tau_j^{(i)}-\lambda_j^{(i)}u_j^{(i)}+c_j^{(i)}\right) x_j^{(i)}\right. \\
& \left.+\left(-\boldsymbol{\nu}^{(i+1) \top} \mathbf{W}_{:, j}^{(i+1)}-\mu_j^{(i)}-\tau_j^{(i)}+(u_j^{(i)} - l_j^{(i)})\lambda_j^{(i)}\right) \hat{x}_j^{(i)} \right] \\
& -\sum_{i=1}^L \boldsymbol{\nu}^{(i) \top} \mathbf{b}^{(i)}+\sum_{i=1}^{L-1} \sum_{j \in \mathcal{I}^{\pm(i)}} \lambda_j^{(i)} u_j^{(i)}l_j^{(i)} \\
\text { s.t. }  \bs{l}^{(0)} & \leq \hat{\bsx}^{(0)} \leq \bs{u}^{(0)} \\
\boldsymbol{\mu} & \geq 0 ; \quad \boldsymbol{\tau} \geq 0 ; \quad \boldsymbol{\gamma} \geq 0 ; \quad \boldsymbol{\lambda} \geq 0
\end{aligned}$$

To minimize  $\left(\bs{c}^{(0)\top}-\boldsymbol{\nu}^{(1) \top} \mathbf{W}^{(1)}\right)\hat{\bsx}^{(0)}$ subject to $\bs{l}^{(0)} \leq \hat{\bsx}^{(0)} \leq \bs{u}^{(0)}$, we can consider the choice we must make in each dimension. If the $j$th entry of the coefficient is positive, we should set $\hat{x}^{(0)}_j = l^{(0)}_j$. Otherwise, we should set $\hat{x}^{(0)}_j = u^{(0)}_j$.

$$\begin{aligned}
\max_{\boldsymbol{\nu}, \boldsymbol{\mu}, \boldsymbol{\tau}, \boldsymbol{\gamma}, \boldsymbol{\lambda}} \min_{\bsx, \hat{\bsx}} \quad &  \left(\boldsymbol{\nu}^{(L)}+\boldsymbol{\gamma}\right)^{\top} \bsx^{(L)}+\left[\bs{c}^{(0)\top} - \boldsymbol{\nu}^{(1) \top} \mathbf{W}^{(1)}\right]_+\bs{l}^{(0)}-\left[\bs{c}^{(0)\top} -\boldsymbol{\nu}^{(1) \top} \mathbf{W}^{(1)}\right]_-\bs{u}^{(0)} \\
& +\sum_{i=1}^{L-1} \sum_{j \in \mathcal{I}^{+(i)}}\left(\nu_j^{(i)}-\boldsymbol{\nu}^{(i+1) \top} \mathbf{W}_{:, j}^{(i+1)} + c^{(i)}_j\right) x_j^{(i)} +\sum_{i=1}^{L-1} \sum_{j \in \mathcal{I}^{-(i)}}\left(\nu_j^{(i)} + c^{(i)}_j\right) x_j^{(i)} \\
& +\sum_{i=1}^{L-1} \sum_{j \in \mathcal{I}^{\pm(i)}}\left[\left(\nu_j^{(i)}+\tau_j^{(i)}-\lambda_j^{(i)}u_j^{(i)}+c_j^{(i)}\right) x_j^{(i)}\right. \\
& \left.+\left(-\boldsymbol{\nu}^{(i+1) \top} \mathbf{W}_{:, j}^{(i+1)}-\mu_j^{(i)}-\tau_j^{(i)}+(u_j^{(i)} - l_j^{(i)})\lambda_j^{(i)}\right) \hat{x}_j^{(i)} \right] \\
& -\sum_{i=1}^L \boldsymbol{\nu}^{(i) \top} \mathbf{b}^{(i)}+\sum_{i=1}^{L-1} \sum_{j \in \mathcal{I}^{\pm(i)}} \lambda_j^{(i)} u_j^{(i)}l_j^{(i)} \\
\text { s.t. } \boldsymbol{\mu} & \geq 0 ; \quad \boldsymbol{\tau} \geq 0 ; \quad \boldsymbol{\gamma} \geq 0 ; \quad \boldsymbol{\lambda} \geq 0
\end{aligned}$$

From here, we note that the variables $\bs{x}$ or $\hat{\bs{x}}$ are unconstrained variables. Therefore, if any of their coefficients are nonzero, the inner minimization can immediately drive its value to $-\infty$. As such, the outer maximization must set all of these coefficients to zero. Therefore, we can derive constraints from this restructured optimization and remove the free variables $\bs{x}, \hat{\bs{x}}$.

$$\begin{aligned}
\max_{\boldsymbol{\nu}, \boldsymbol{\mu}, \boldsymbol{\tau}, \boldsymbol{\gamma}, \boldsymbol{\lambda}} \quad & \left[\bs{c}^{(0)\top} - \boldsymbol{\nu}^{(1) \top} \mathbf{W}^{(1)}\right]_+\bs{l}^{(0)}-\left[\bs{c}^{(0)\top} -\boldsymbol{\nu}^{(1) \top} \mathbf{W}^{(1)}\right]_-\bs{u}^{(0)} -\sum_{i=1}^L \boldsymbol{\nu}^{(i) \top} \mathbf{b}^{(i)}\\
&+ \sum_{i=1}^{L-1} \sum_{j \in \mathcal{I}^{\pm(i)}}\lambda_j^{(i)} u_j^{(i)}l_j^{(i)} \\
\text { s.t. } \quad 
\boldsymbol{\nu}^{(L)} &= -\boldsymbol{\gamma} \\
\nu_j^{(i)} &= \boldsymbol{\nu}^{(i+1) \top} \mathbf{W}_{:, j}^{(i+1)} - c^{(i)}_j, j \in \mathcal{I}^{+(i)} \\
\nu_j^{(i)} &= - c^{(i)}_j, j \in \mathcal{I}^{-(i)} \\
\nu_j^{(i)} &= \lambda_j^{(i)}u_j^{(i)}-\tau_j^{(i)}-c_j^{(i)}, j \in \mathcal{I}^{\pm(i)} \\
\boldsymbol{\nu}^{(i+1) \top} \mathbf{W}_{:, j}^{(i+1)} &= (u_j^{(i)} - l_j^{(i)})\lambda_j^{(i)}-\left(\mu_j^{(i)}+\tau_j^{(i)}\right), j \in \mathcal{I}^{\pm(i)} \\
\boldsymbol{\mu} & \geq 0 ; \quad \boldsymbol{\tau} \geq 0 ; \quad \boldsymbol{\gamma} \geq 0 ; \quad \boldsymbol{\lambda} \geq 0
\end{aligned}$$

For the following, we define $\hat{\nu}^{(i)}_j = \boldsymbol{\nu}^{(i+1) \top} \mathbf{W}_{:, j}^{(i+1)}$. We note that since the upper and lower bounds of the neuron relaxation can not be tight simulataneously, at least one of $(u_j^{(i)} - l_j^{(i)})\lambda^{(i)}_j$ and $\mu^{(i)}_j + \tau^{(i)}_j$ must be non-zero. Therefore, we can write them as $(u_j^{(i)} - l_j^{(i)})\lambda^{(i)}_j = [\hat{\nu}^{(i)}_j]_+ $ and $ \mu^{(i)}_j + \tau^{(i)}_j = [\hat{\nu}^{(i)}_j]_-$. We can then use the fact that $\tau^{(i)}_j$ lies in the interval $0$ and $\hat{\nu}^{(i)}_j$ to get the following bound propagation procedure.

$$\begin{aligned}
\max_{\boldsymbol{\nu}, \bs{\alpha}, \boldsymbol{\gamma}} \quad & \left[\bs{c}^{(0)\top} - \boldsymbol{\nu}^{(1) \top} \mathbf{W}^{(1)}\right]_+\bs{l}^{(0)}-\left[\bs{c}^{(0)\top} -\boldsymbol{\nu}^{(1) \top} \mathbf{W}^{(1)}\right]_-\bs{u}^{(0)} -\sum_{i=1}^L \boldsymbol{\nu}^{(i) \top} \mathbf{b}^{(i)}\\
& + \sum_{i=1}^{L-1} \sum_{j \in \mathcal{I}^{\pm(i)}} \left[\frac{u^{(i)}_{j}l^{(i)}_j [\hat{\nu}^{(i)}_j]_{+}}{u^{(j)}_{i} - l^{(j)}_i}\right] \\
\text { s.t. } \quad
\boldsymbol{\nu}^{(L)} &= -\boldsymbol{\gamma} \\
\nu_j^{(i)} &= \boldsymbol{\nu}^{(i+1) \top} \mathbf{W}_{:, j}^{(i+1)} - c^{(i)}_j, j \in \mathcal{I}^{+(i)} \\
\nu_j^{(i)} &= - c^{(i)}_j, j \in \mathcal{I}^{-(i)} \\
\hat{\nu}_j^{(i)} &= \boldsymbol{\nu}^{(i+1) \top} \mathbf{W}_{:, j}^{(i+1)} \\
\nu_j^{(i)} &= \frac{u^{(i)}_{j}}{u^{(j)}_{i} - l^{(j)}_i}[\hat{\nu}^{(i)}_j]_{+} - \alpha^{(i)}_j[\hat{\nu}^{(i)}_{j}]_{-} - c^{(i)}_j, j \in \mathcal{I}^{\pm(i)} \\
\alpha^{(i)}_j & \in [0, 1] ; \quad \boldsymbol{\gamma} \geq 0
\end{aligned}$$

In this program, $\alpha^{(i)}_j$ are optimizable parameters controlling the relaxation of neuron $j$ in layer $i$, similar to the ones appearing in \citealt{xu2021fast}. $\bs{\gamma}$, as discussed in the body of the paper, is the parameter which enforces the output constraint throughout this entire bound propagation procedure. 

\section{Connection to Forward Verification for Intermediate Bounds}\label{sec:skip-connection}

We consider the general objective presented in \ref{sec:intermediate-bounds-theorem} and aim to minimize $\bs{c}^{(0)\top}\hat{\bs{x}}^{(0)} + \sum_{i=1}^{L-1}\bs{c}^{(i)\top}\bs{x}^{(i)}$

$$\begin{aligned}
    \min_{\bsx} \quad & \bs{c}^{(0)\top}\hat{\bs{x}}^{(0)} + \sum_{i=1}^{L-1}\bs{c}^{(i)\top}\bs{x}^{(i)} \\
    \text{s.t.} \quad
    & \bsx \in \mathcal{X} ; \quad \mathbf{H}f(\bsx) + \mathbf{d} \leq 0
\end{aligned}$$

If we take the dual of the constraint, then we get the program

$$\begin{aligned}
    \min_{\bsx} \quad & \bs{c}^{(0)\top}\hat{\bs{x}}^{(0)} + \sum_{i=1}^{L-1}\bs{c}^{(i)\top}\bs{x}^{(i)} + \bs{\gamma}\left(\mathbf{H}f(\bsx) + \mathbf{d}\right)  \\
    \text{s.t.} \quad
    & \bsx \in \mathcal{X} ; \quad \bs{\gamma} \geq 0
\end{aligned}$$

We note that the objective here can be expressed as a neural network with residual stream $f(\bs{x})$ and a skip connection with linear weight $\bs{c}^{(i)}$ from the pre-activations of layer $i$ to the final output for every layer. In theory, we could construct this neural network and directly pass it to a forward verification tool which could iteratively tighten all bounds for a solution.

\section{Implementation}
\label{sec:implementation}
As stated in Algorithm \ref{alg:invprop}, INVPROP first initializes bounds for all layers using some computationally cheap technique.
Based on the input bounds given by $\mathcal{X}$, we first compute intermediate bounds using interval propagation and then tighten them based on the reverse symbolic interval propagation (RSIP) technique \cite{singh2019rsip}.
While INVPROP will iteratively tighten those bounds over time, we found RSIP to reduce the total necessary runtime significantly compared to an initialization based solely on interval propagation.
We initialize all $\bs{\alpha}$ with $0.5$ and all $\bs{\gamma}$ with $0.025$.

The optimization is performed for all lower and upper bounds of all neurons in each layer in parallel, starting with the last layer and moving forward to the input layer. The bounds of the cutting hyperplanes are optimized last, together with the bounds on the input neurons.
The improvements on $\bs{c}^{\top}\bsx$ are measured every 10th iteration. Before doing so, the bounds of the hyperplanes are tightened for 10 extra steps to improve their precision.
We detect convergence by monitoring $\bs{c}^{\top}\bsx$.
If successive iterations see minimal improvement, we stop or branch on the input space.
All cutting hyperplanes are evenly distributed to maximize their information gain.
For 40 hyperplanes in a 2D input space, we rotate each plane by 9\textdegree.
All reported runtimes under $10$ minutes are computed as the average of five runs.

\subsection{Encoding Non-Linear Constraints}
\label{sec:encodemax}
To support non-linear output sets, such as the maximum operation used in the OOD example in Section \ref{sec:results-ood}, the non-linearity needs to be encoded, such that it can be expressed as linear constraints over the modified network.
We rewrite $\max(y_1, y_2) = \max(y_1 - y_2, 0) + y_2$ and add an additional ReLU layer for this operation.
Note that to pass $y_2$ through this layer without modifying it, one can either write $y_2 = \max(y_2, 0) - \max(-y_2, 0)$, or $y_2 = \max(y_2 - M, 0) + M$ where $M$ is the lower bound of all possible $y_2$. 
We find that avoiding the additional ReLU relaxations that would occur for the first approach is beneficial for the optimization and compute a lower bound of $M$ using interval propagation.
As $y_3$ (for the OOD class) should not be changed by this operation either, we apply the same trick of subtracting and adding its lower bound.

\subsection{Control Benchmark Encoding}
\label{sec:controlbenchmarkencoding}
We encode the entire control formula $\bs{x} = \mathbf{A}\bs{x} + \mathbf{B}\bs{u}$ as one feedforward network by encoding the residual connection as regular fully connected layers.
To this end, we use the same technique described in Section \ref{sec:encodemax} to shift the bounds into the positive regime, feed them forward and then shift them back.

To compute $\Sover$ for a timestep $t > 1$, we first stack the network $t$ times, then simplify it by merging consecutive linear layers.
All bounds of layers not affected by this merging that also appear in the network for $t-1$ timesteps are reused.
Their bounds are already tight enough and are not optimized further.
Note that this is different than using the previous $\Sover$ as the new target region and using an unstacked network:
All bounds of the new layers are still optimized w.r.t. the precise target $\Sout$. 
Therefore, we do not suffer from accumulating inaccuracies.

\section{Hardware}
\label{sec:hardwar}
For the control benchmark, all experiments were performed on a Dual Xeon Gold 6138.

For the OOD benchmark, all experiments were performed on an Intel Xeon Platinum 8160 processor using 8 cores and 40GB RAM, as well as a V100-SXM2 GPU.

For the MNIST and YOLO robostness verification benchmarks, an AWS instance of type g5.2xlarge was used, with is equipped with a AMD EPYC 7R32 processor with 8 cores, 32GB RAM, as well as an A10G GPU.

\section{Robustness Verification}
\label{sec:detailedrobustnessbenchmarks}
For the comparison in Section~\ref{sec:evalnnverification}, we have implemented the concept of INVPROP in $\abcrown$.

Figure~\ref{fig:mnistbenchmark} compares the performance of $\abcrown$ and $\abcrown$ + INVPROP on the MNIST benchmark defined by \cite{pengfei2021improving}.

\begin{figure}
    \centering
    \includegraphics[trim={0 2.3cm 0 10cm},clip,width=0.7\textwidth]{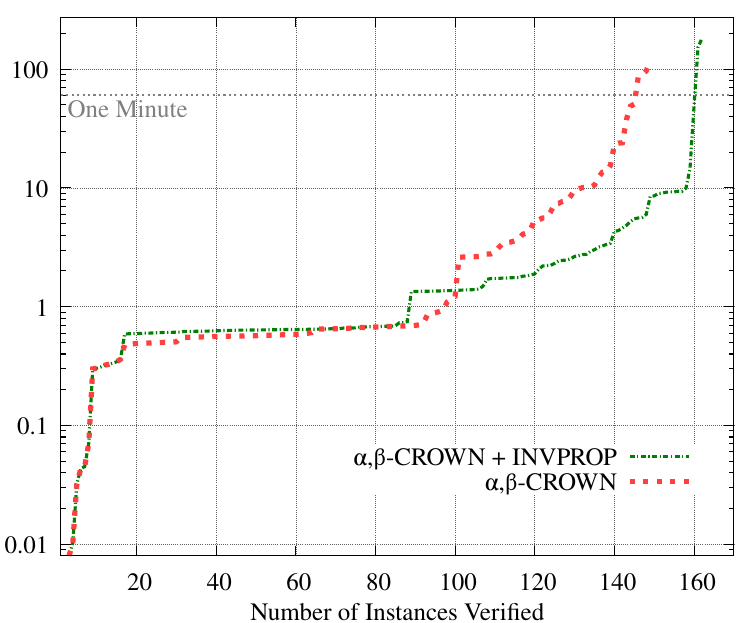}
    \caption{Detailed runtime comparison of $\abcrown$ and $\abcrown$ + INVPROP. Except for a few instances, $\abcrown$ extended with INVPROP can verify the properties faster, and can prove robustness for some instances that cause a timeout for pure $\abcrown$. The timeout per instance is 5 minutes.}
    \label{fig:mnistbenchmark}
\end{figure}

\end{document}